\documentclass[twoside]{article}

\usepackage[preprint]{aistats2027}
\usepackage{booktabs}
\usepackage{multirow}
\usepackage{makecell}

\usepackage{amsmath,amssymb,amsthm,mathtools}
\usepackage{microtype}
\usepackage{enumitem}
\usepackage[hidelinks]{hyperref}
\usepackage{cleveref}
\usepackage{amsthm}
\newtheorem{corollary}{Corollary}
\newtheorem{proposition}{Proposition}
\newtheorem{lemma}{Lemma}
\newtheorem{theorem}{Theorem}

\usepackage{bm}
\usepackage{pgfplots}
\usepackage{subcaption}
\pgfplotsset{compat=1.18}
\usepackage{tikz}
\usetikzlibrary{positioning}
\usetikzlibrary{matrix,calc}
\usepgfplotslibrary{colormaps}
\usepackage{color}

\usetikzlibrary{calc}
\usepackage{graphicx}
\usepgfplotslibrary{groupplots}

\newcommand{\qtext}[1]{\quad\text{#1}\quad}

\newcommand{\E}{\mathbb{E}}

\newcommand{\defeq}{\coloneqq}

\newenvironment{talign*}
 {\let\displaystyle\textstyle\csname align*\endcsname}
 {\endalign}
\newenvironment{talign}
 {\let\displaystyle\textstyle\csname align\endcsname}
 {\endalign}

\usepackage[round]{natbib}

\begin{document}

\runningauthor{Buu Phan, Lester Mackey, David Alvarez-Melis}

\twocolumn[

\aistatstitle{Universal Bounds for Self-Normalized Kullback--Leibler Estimation}

\aistatsauthor{
  Buu Phan$^{1\dagger}$
  \And
  Lester Mackey$^{2}$
  \And
  David Alvarez-Melis$^{2,3}$
}

\aistatsaddress{  $^{1}$University of Toronto  \And  $^{2}$Microsoft Research \And $^{3}$Harvard University} ]

\begin{abstract}
This work studies an entropy-based estimator of Kullback--Leibler (KL) divergence constructed from normalized importance weights. We prove a universal finite-sample bound showing that its expectation can exceed the target–proposal KL divergence by at most 
$\log 2$ for any sample size. We then characterize the KL divergence between the distribution induced by importance-weighted resampling and the proposal, showing that it is upper bounded by both the expectation of the entropy-based estimator and the target--proposal KL divergence. Importantly, these bounds require neither bounded importance weights nor finite higher moments. We apply these results to test-time alignment of generative models via Soft Best-of-$N$ sampling, an importance-sampling-based scaling technique proposed by \cite{verdun2025soft}, and to relative entropy coding, obtaining finite-sample policy-divergence certificates and moment-free bounds on coding cost.
\end{abstract}

\section{Introduction}
Importance sampling is a classical Monte Carlo technique for estimating
expectations under a target distribution $p$ using samples from a more
accessible proposal distribution $q$ \citep{robert1999monte}. Throughout,
$p$ and $q$ denote probability densities or mass functions, as appropriate.
Suppose that $p$ and $q$ have the same support, and define the positive
likelihood ratio $w(x)\defeq p(x)/q(x)$. Given $X_1,\ldots,X_N\overset{\mathrm{i.i.d.}}{\sim}q$, importance sampling
estimates the expected value of a function $f(X)$ under the target
distribution $p$ by reweighting samples from the proposal distribution $q$. In particular, we have:
\begin{talign*}
\E_p[f(X)]
=
\E_q[w(X)f(X)]
\approx
\frac{1}{N}\sum_{i=1}^N w(X_i)f(X_i).
\end{talign*}

When $p$ is known only up to a normalizing constant, self-normalized
importance sampling (SNIS) instead approximates the target expectation as
\begin{talign}
\E_p[f(X)]
\approx
\sum_{i=1}^N \lambda_i f(X_i),
\label{eq:snis}
\end{talign}
where
\(
\lambda_i \defeq w(X_i)\big/\sum_{j=1}^N w(X_j) > 0.
\)
Because common multiplicative factors cancel, SNIS is computable without
the target normalizing constant, and its normalized weights naturally define
a categorical distribution for importance-weighted resampling. Under standard
regularity conditions, the SNIS estimator converges to the target expectation
as $N\to\infty$, see \cite{mcbook}. At finite $N$, however, the random normalization couples the
samples, making nonasymptotic guarantees substantially more challenging
~\citep{agapiou2017importance}. In this work, we study the following normalized-weight estimator of the
target--proposal Kullback-Leibler (KL) divergence $D_{\mathrm{KL}}(p\|q)\defeq \E_p[\log(p(X)/q(X))]$:
\[
\widehat D_N
{\defeq}
\sum_{i=1}^N \lambda_i\log(N\lambda_i)
{=}
D_{\mathrm{KL}}(\bm{\lambda}\|U_N)
{=}
\log N-H(\bm{\lambda}),
\]
where $\bm{\lambda}\defeq(\lambda_1,\ldots,\lambda_N)$, $U_N$ is uniform on
$\{1,\ldots,N\}$, and $H(\cdot)$ denotes entropy. We use natural logarithms throughout unless stated otherwise. This estimator has previously appeared
as an entropy or perplexity criterion in adaptive importance sampling and
sequential Monte Carlo \citep{cappe2008adaptive} and satisfies
\newcommand{\toas}{\stackrel{a.s.}{\to}}
\[
\widehat D_N \toas D_{\mathrm{KL}}(p\|q)
\qtext{and}
\E[\widehat D_N] \to D_{\mathrm{KL}}(p\|q)
\]
when $D_{\mathrm{KL}}(p\|q)<\infty$; see
Appendix~\ref{sec:asymptotic}. Unlike the direct SNIS estimator obtained by setting
$f(X)=\log w(X)$ in \Cref{eq:snis}, $\widehat D_N$ depends only on the
normalized weights and is therefore computable even when the target
normalizing constant is unknown. Moreover,
\[
0\leq \widehat D_N\leq\log N,
\]
since it is the KL divergence between $\bm{\lambda}$ and the uniform
distribution on $N$ points. Consequently, one immediately obtains
$\E[\widehat D_N]\leq\log N$, where the expectation is taken over the
proposal samples $X_{1:N}\sim q^N$ unless otherwise stated.  However, this bound grows with $N$ and does
not reflect the discrepancy between $p$ and $q$. Sharper nonasymptotic
guarantees for SNIS quantities typically require additional assumptions,
such as bounded importance weights or finite higher moments
\citep{agapiou2017importance,phan2024importance}. In contrast, we show that $\E[\widehat D_N]$ admits a universal finite-sample bound even in the absence of these additional assumptions.

\paragraph{Main Results.}
As stated above, we assume throughout that $p$ and $q$ share the same
support and that $D_{\mathrm{KL}}(p\|q)$ is finite, but impose no
boundedness or additional moment assumptions.
Our first result shows that, for any sample size $N$, $\E[\widehat D_N]$
cannot exceed $D_{\mathrm{KL}}(p\|q)$ by more than a universal additive
constant, while converging to $D_{\mathrm{KL}}(p\|q)$ as $N\to\infty$.

\begin{theorem}[Universal Upper Bound]
\label{prop:kl_estimator}
For every $\quad \, N \geq 1$,
\[
\E[\widehat D_N]
<
D_{\mathrm{KL}}(p\|q)+\log 2.
\]
\end{theorem}

\noindent
Despite its simplicity, to the best of our knowledge, this bound was previously unknown; its proof is given in Appendix~\ref{app:b}.  The additive $\log 2 \approx 0.69$ term results from a worst-case analysis that combines several inequalities whose worst cases need not occur simultaneously and is therefore unlikely to be tight; our exhaustive experimental search in Appendix \ref{sec:exp} suggests that the optimal universal slack may instead be near $0.15$.

We next consider the distribution induced by SNIS resampling. For each
output, draw a fresh proposal set $X_{1:N}$ and sample
\[
I\mid X_{1:N}\sim\operatorname{Cat}(\bm{\lambda}).
\]
Let $p_N$ denote the distribution of $X_I$ (which depends on the randomness of both
the proposal set and the resampling step). We show that
$D_{\mathrm{KL}}(p_N\|q)$ is upper bounded by both
$\E[\widehat D_N]$ and $D_{\mathrm{KL}}(p\|q)$. 

\begin{theorem}[KL Contraction under SNIS]
\label{prop:2nd_result}
Suppose $D_{\mathrm{KL}}(p\|q)<\infty$. Then, for every $N\geq1$,
\[
D_{\mathrm{KL}}(p_N\|q)
\leq
\min\!\big\{
\E[\widehat D_N],
D_{\mathrm{KL}}(p\|q)
\big\}.
\]
\end{theorem}
The proof of \Cref{prop:2nd_result} is given in
Appendix~\ref{thm:2}. The bound
$D_{\mathrm{KL}}(p_N\|q)\leq D_{\mathrm{KL}}(p\|q)$
is intuitively plausible since the candidates are drawn from $q$, but establishing it requires constructing a nontrivial contraction lemma. Meanwhile, $\E[\widehat D_N]$ provides a straightforwardly estimated upper
bound on $D_{\mathrm{KL}}(p_N\|q)$ 
using only the normalized importance
weights, without requiring explicit evaluation of $p_N$. This is particularly
useful in applications such as model alignment, where the divergence of the
finite-sample resampled policy from the proposal is of direct interest. Importantly, \Cref{prop:kl_estimator} further guarantees that
$\E[\widehat D_N]$ cannot substantially exceed the population-level
divergence. This property is especially valuable when 
$N \gg \exp({D_{\mathrm{KL}}(p\|q)})$ as  
the crude $\log N$ upper bound then becomes unboundedly larger than $D_{\mathrm{KL}}(p\|q)$. 

We complement this upper bound with a straightforwardly estimated lower bound, providing
a practical way to approximate $D_{\mathrm{KL}}(p_N\|q)$ without explicitly
evaluating $p_N$. For $Y\sim p_N$ and $M\geq1$, define
\[
\widehat\theta_M(Y)
\defeq
\frac{N}{M}\sum_{m=1}^M
\frac{w(Y)}
{w(Y)+\sum_{j=2}^N w(X_j^{(m)})},
\]
where $X_j^{(m)}\overset{\mathrm{i.i.d.}}{\sim}q$ are auxiliary samples
independent of $Y$. Here, $\widehat\theta_M(Y)$ is a Monte Carlo estimator
of the density ratio $p_N(Y)/q(Y)$, obtained by averaging the SNIS selection
probability of $Y$ over fresh auxiliary proposal sets.

\begin{theorem}[KL Lower Bound]
\label{prop:snis-lower}
For all $N,M\geq1$,
\[
\E_{Y\sim p_N}\!\left[
\E_{X_{2:N}^{(1:M)}\sim q}
\left[\log\widehat\theta_M(Y)\right]
\right]
\leq
D_{\mathrm{KL}}(p_N\|q),
\]
where the inner expectation is over the $M$ independent auxiliary proposal
sets used to construct $\widehat\theta_M(Y)$. Moreover, for independent $Y_1,\ldots,Y_B\sim p_N$, with independent
auxiliary proposal sets for each $Y_b$, define
\begin{talign}
\widehat L_{B,M}
\defeq
\max(0,\frac1B\sum_{b=1}^B\log\widehat\theta_M(Y_b)) .
\end{talign}
Then, almost surely,
\[
\lim_{B\to\infty}\lim_{M\to\infty}\widehat L_{B,M}
=
D_{\mathrm{KL}}(p_N\|q).
\]
\end{theorem}
The proof is shown in Appendix \ref{thm3:proof}. Together with \Cref{{prop:2nd_result}}, this result provides a practical way to
bracket and approximate $D_{\mathrm{KL}}(p_N\|q)$ without explicitly
evaluating $p_N$. As $M \to \infty$ followed by $B \to \infty$, the lower-bound estimator
converges almost surely to the true divergence. Clipping at zero ensures that the estimator, like the $D_{\mathrm{KL}}$ quantity being estimated, is always nonnegative. 
Related auxiliary-sampling estimators have previously been used to evaluate implicit importance-weighted distributions in 
variational autoencoder literature \citep{cremer2017reinterpreting}.

We empirically evaluate our bounds in Appendix~\ref{sec:exp} and illustrate their implications in two modern applications of importance sampling: test-time alignment via Soft Best-of-$N$ sampling and approximate relative entropy coding. In both settings, our results require neither bounded importance weights
nor higher-order moment conditions, while yielding tighter bounds across
relevant regimes.

\section{Divergence Guarantees for Soft Best-of-\texorpdfstring{$N$}{N}}

\subsection{Test-Time Alignment with KL Control} 
A central goal of language-model alignment is to steer a base policy toward high-reward outputs while controlling distributional shift from the original model. Test-time alignment \citep{stiennon2020learning, huang2025best} provides an attractive alternative to fine-tuning by leaving the base policy fixed and using additional computation at generation
time to favor high-reward responses. Since aggressive reward optimization can move the resulting distribution substantially away from the base policy, KL regularization provides a natural mechanism for trading off reward improvement
against fidelity to the original model.

Concretely, suppose we have a base policy $\pi_0(y\mid x)$ from which we
can sample candidate responses and a reward function
$r(x,y)\in\mathbb{R}$ that measures the desirability of response $y$ for
input $x$. The reward function may encode human preferences, correctness,
task performance, or other desired behavior. Our goal is to generate high-reward responses while controlling the
distributional shift from the base policy
\citep{csiszar2003information,verdun2025soft}. For each prompt $x$, such policy is the solution of the following KL-constrained optimization problem:
\[
\begin{aligned}
\max_{\pi}\quad
& \E_{Y\sim\pi(\cdot\mid x)}[r(x,Y)] \\
\text{s.t.}\quad
& D_{\mathrm{KL}}\!\left(
\pi(\cdot\mid x)\|\pi_0(\cdot\mid x)
\right)\leq \epsilon.
\end{aligned}
\]
When the constraint is active, the optimal population Gibbs policy is
an exponential tilt of the base policy:
\begin{equation}
\pi_\beta(y\mid x)
=
\frac{\pi_0(y\mid x)e^{\beta r(x,y)}}
{\E_{Y\sim\pi_0(\cdot\mid x)}[e^{\beta r(x,Y)}]}, \label{tilted_policy}    
\end{equation}
where $\beta>0$ is the Lagrange parameter associated with the KL constraint
and determines a point on the reward--KL trade-off
(see \citealp{verdun2025soft} and Appendix~\ref{sBoN_tilted_policy} for a
detailed derivation).
The normalizing constant in the denominator is
generally intractable, so direct sampling from $\pi_\beta$ is not available.
Thus, the test-time alignment goal is to approximately sample from 
the reward-tilted policy $\pi_\beta$ using only samples from the base policy
$\pi_0$ and evaluations of the reward $r(x,y)$. As a result, for a given choice of $\beta$, test-time alignment naturally gives
rise to three computational objectives: approximating the population Gibbs
policy $\pi_\beta$ with a policy $\pi_{N,\beta}$ from which we can sample
using $N$ draws from $\pi_0$, evaluating its corresponding expected reward,
\(
\mathbb{E}_{Y\sim\pi_{N,\beta}(\cdot\mid x)}[r(x,Y)],
\)
and quantifying the population distributional shift of this approximation,
\(
D_{\mathrm{KL}}\!\left(
\pi_{N,\beta}(\cdot\mid x)
\middle\|
\pi_0(\cdot\mid x)
\right).
\)
We construct such a policy using self-normalized importance sampling,
a procedure known in this setting as Soft Best-of-$N$.

\begin{figure*}[t]
\centering
\begin{tikzpicture}

\begin{scope}[
    xshift=0cm,
    yshift=2.05cm,
    scale=0.78,
    transform shape,
    >=stealth,
    box/.style={
        draw,
        rectangle,
        minimum height=0.78cm,
        minimum width=2.35cm,
        align=center,
        font=\small
    },
    arrow/.style={->, thick},
    every node/.style={font=\small}
]

\node[box] (base) at (0,2.15)
    {Base policy $\pi_0$};

\node[box] (pool) at (0,-2.15)
    {$Y_{1:N}\sim\pi_0$};

\node[box, fill=orange!20] (gibbs) at (4.8,2.15)
    {$\pi_\beta \propto \pi_0(y\mid x)e^{\beta r(x,y)}$};

\node[box] (reward) at (4.8,0)
    {Reward $r(x,y)$};

\node[box, fill=blue!15, minimum width=2.7cm] (weights) at (4.8,-2.15)
    {$\lambda_i\propto e^{\beta r(x,Y_i)}$};

\node[box, minimum width=0.9cm] (output) at (8.8,-2.15)
    {$Y_I$};

\node at (4.8,3.35)
    {Target Gibbs Policy};

\node at (4.8,-3.35)
    {Soft Best-of-$N$};

\draw[arrow]
    (base.east) -- (gibbs.west);

\draw[arrow]
    (base.south)
    --
    node[left, xshift=-2pt] {$N$ samples}
    (pool.north);

\draw[arrow]
    (reward.north) -- (gibbs.south);

\draw[arrow]
    (reward.south) -- (weights.north);

\draw[arrow]
    (pool.east) -- (weights.west);

\draw[arrow]
    (weights.east)
    --
    node[above, yshift=2pt] {$I\sim\operatorname{Cat}(\lambda)$}
    (output.west);

\end{scope}

\begin{axis}[
    at={(8.65cm,0.18cm)},
    anchor=south west,
    width=7.05cm,
    height=4.65cm,
    tick label style={font=\scriptsize},
    label style={font=\scriptsize},
    xmin=1,
    xmax=32,
    ymin=0,
    ymax=3.2,
    xlabel={Number of candidates $N$},
    ylabel={KL divergence (nats)},
    xtick={1,8,16,24,32},
    ytick={0,0.5,1,1.5,2,2.5,3},
    ymajorgrids=true,
    xmajorgrids=false,
    grid style={black!10},
    axis lines=left,
    tick align=outside,
    clip=true,
    legend style={
        draw=none,
        fill=none,
        font=\scriptsize,
        at={(0.5,1.035)},
        anchor=south,
        legend columns=2,
        /tikz/every even column/.append style={column sep=5pt},
        row sep=0.5pt,
    },
]

\addplot+[
    blue!75!black,
    mark=*,
    mark size=0.7pt,
    line width=0.9pt,
]
coordinates {
    (1,0.000000)
    (2,0.167078)
    (3,0.349687)
    (4,0.479162)
    (5,0.567780)
    (6,0.628655)
    (7,0.671069)
    (8,0.701180)
    (9,0.723005)
    (10,0.739171)
    (11,0.751405)
    (12,0.760862)
    (13,0.768321)
    (14,0.774316)
    (15,0.779220)
    (16,0.783837)
    (18,0.791811)
    (20,0.798343)
    (24,0.808071)
    (28,0.814593)
    (32,0.818965)
};
\addlegendentry{
    $D_{\mathrm{KL}}(\pi_{N,\beta}\Vert\pi_0)$
}

\addplot+[
    orange!90!black,
    dashed,
    mark=square*,
    mark size=0.65pt,
    line width=0.9pt,
]
coordinates {
    (1,0.000000)
    (2,0.422441)
    (3,0.609278)
    (4,0.705393)
    (5,0.759212)
    (6,0.790814)
    (7,0.809894)
    (8,0.821597)
    (9,0.828823)
    (10,0.833278)
    (11,0.835991)
    (12,0.837598)
    (13,0.838498)
    (14,0.838944)
    (15,0.839100)
    (16,0.837765)
    (18,0.835548)
    (20,0.833821)
    (24,0.831459)
    (28,0.830027)
    (32,0.829159)
};
\addlegendentry{
    $\mathbb{E}[\widehat D_N]$ (Our)
}

\addplot+[
    black,
    dashed,
    mark=none,
    line width=0.85pt,
]
coordinates {
    (1,0.827735)
    (32,0.827735)
};
\addlegendentry{
    $D_{\mathrm{KL}}(\pi_\beta\Vert\pi_0)$
}

\addplot+[
    red!65!black,
    dashdotted,
    mark=none,
    line width=0.95pt,
]
coordinates {
    (1,1.520882)
    (32,1.520882)
};
\addlegendentry{
    $D_{\mathrm{KL}}(\pi_\beta\Vert\pi_0)+\log 2$
}

\addplot+[
    green!55!black,
    solid,
    mark=none,
    line width=1.05pt,
    domain=1:32,
    samples=220,
]
{ln(x/(1+(x-1)*exp(-12)))};
\addlegendentry{Aminian et al.}

\node[
    green!45!black,
    anchor=south east,
    font=\scriptsize,
    fill=white,
    inner sep=1.2pt
] at (axis cs:31.3,2.78)
    {$\to \beta R_{\max}$};

\end{axis}

\end{tikzpicture}

\caption{
Soft Best-of-$N$ and its KL guarantees. The population KL is the limiting
SBoN divergence as $N\to\infty$. Our estimator $\E[\widehat D_N]$ is
tighter than the guarantee of \citet{aminian2026best} and
can overshoot the population KL by at most $\log 2$. In contrast, the bound of \citet{aminian2026best} depends only on the reward
range and does not account for the geometry of the reward distribution under
the base policy. It converges to $\beta R_{\max}$, which is $12$ in this
example, corresponding to the first column and last row of
\Cref{fig:softbon-alignment}.
}
\label{fig:softbon}
\end{figure*}

\subsection{Soft Best-of-\texorpdfstring{$N$}{N} Alignment and KL Guarantees}
The Soft Best-of-$N$ (SBoN) algorithm~\citep{verdun2025soft} 
provides a SNIS approximation to $\pi_\beta$.
Specifically, SBoN draws
\[
Y_1,\ldots,Y_N
\overset{\mathrm{i.i.d.}}{\sim}
\pi_0(\cdot\mid x)
\]
and assigns each candidate the normalized weight
\[
\begin{aligned}
\lambda_i
&=
\frac{\pi_\beta(Y_i\mid x)/\pi_0(Y_i\mid x)}
{\sum_{j=1}^N \pi_\beta(Y_j\mid x)/\pi_0(Y_j\mid x)}
=
\frac{e^{\beta r(x,Y_i)}}
{\sum_{j=1}^N e^{\beta r(x,Y_j)}},
\end{aligned}
\]
for $i\in[N]$. It then samples
\[
I\mid Y_{1:N}
\sim
\operatorname{Cat}(\lambda_1,\ldots,\lambda_N)
\]
and returns $Y_I$. We denote the law of the returned response $Y_I$ by
$\pi_{N,\beta}(\cdot\mid x)$. Equivalently, we have:
\begin{align}
&\pi_{N,\beta}(y\mid x) \notag
\\&=
\mathbb{E}_{
Y_{1:N}\sim\pi_0(\cdot\mid x)^{\otimes N}
}
\left[
\frac{
\sum_{i=1}^N
e^{\beta r(x,Y_i)}
\mathbf{1}\{Y_i=y\}
}{
\sum_{j=1}^N
e^{\beta r(x,Y_j)}
}
\right].  \notag  
\end{align}
Moreover, provided that the normalizing constant in \Cref{tilted_policy} is finite,
the finite-$N$ SBoN policy converges to the population Gibbs policy, i.e.
\(
\pi_{N,\beta}(\cdot\mid x)
\longrightarrow
\pi_\beta(\cdot\mid x)\)
as $N\to\infty$. Details regarding the expression of $\pi_{N,\beta}$ and its convergence are provided in Appendix~\ref{sBoN consistency}. 

Since inference is constrained by a finite computational budget $N$, SBoN implements the finite-sample policy $\pi_{N,\beta}$\footnote{In the limit $\beta\to\infty$, SBoN reduces to standard Best-of-$N$ \citep{beirami2024theoretical}, as the resampling distribution concentrates on the highest-reward candidate.} rather than the ideal Gibbs policy $\pi_\beta$. This makes it important to characterize the distributional shift induced by the finite-$N$ procedure, as measured by
\(
D_{\mathrm{KL}}
\left(
\pi_{N,\beta}(\cdot\mid x)
\middle\|
\pi_0(\cdot\mid x)
\right)
\). Prior work by \citet{aminian2026best} provides a worst-case guarantee for this finite-$N$ divergence under the bounded-reward assumption $0\le r(x,y)\le R_{\max}$, showing that
\begin{talign}
\!\!\!\!D_{\mathrm{KL}}\!\left(
\pi_{N,\beta}(\cdot\mid x)
\middle\|
\pi_0(\cdot\mid x)
\right) 
\leq
\log\frac{N}
{1+(N-1)e^{-\beta R_{\max}}}.
\label{bound:r_max}
\end{talign}
As $N\to\infty$, the right-hand side converges to $\beta R_{\max}$,
which upper bounds the population divergence. Indeed,
\begin{align*}
&D_{\mathrm{KL}}\!\left(\pi_\beta(\cdot\mid x)\middle\|
\pi_0(\cdot\mid x)\right)\\
&=
\beta\E_{\pi_\beta(\cdot\mid x)}[r(x,Y)]
-\log\E_{\pi_0(\cdot\mid x)}[e^{\beta r(x,Y)}]
\nonumber\\
&\leq \beta R_{\max}.
\end{align*}
Thus, the bound does not capture the asymptotic behavior of SBoN: as $N\to\infty$, $\pi_{N,\beta}$ converges to $\pi_\beta$, whereas the bound converges to $\beta R_{\max}$, which can be strictly larger than the population divergence. Moreover, since the bound depends only on the reward range and $N$, it does not account for how high-reward responses are distributed under the base policy $\pi_0$, and therefore need not provide a tight characterization of the divergence. 

We next apply our general SNIS results to SBoN, obtaining a tighter, alignment-sensitive bound on the finite-$N$ KL divergence. For each realization
\(
Y_1,\ldots,Y_N\overset{\mathrm{i.i.d.}}{\sim}\pi_0(\cdot\mid x),
\)
using the KL estimator
\(
\widehat D_N
=
\sum_{i=1}^N
\lambda_i\log(N\lambda_i),
\)
we obtain the following guarantees.

\begin{corollary}[KL Guarantees for Soft Best-of-$N$]
\label{cor:sbon}
For every $x$ and $N\geq1$,
\begin{align}
&D_{\mathrm{KL}}\!\left(
\pi_{N,\beta}(\cdot\mid x)
\middle\|
\pi_0(\cdot\mid x)
\right) \notag\\
&\quad\leq
\min\!\left\{
\E[\widehat D_N],\,
D_{\mathrm{KL}}\!\left(
\pi_\beta(\cdot\mid x)
\middle\|
\pi_0(\cdot\mid x)
\right)
\right\}.
\label{eq:kl-sbon}
\end{align}
If $0\leq r(x,y)\leq R_{\max}$, then
\begin{talign}
\widehat D_N
\leq
\log\frac{N}
{1+(N-1)e^{-\beta R_{\max}}}.
\label{eq:sbon-estimator-bound}
\end{talign}
Consequently, $\E[\widehat D_N]$ is no larger than the bounded-reward
guarantee~\eqref{bound:r_max} of \citet{aminian2026best}, so the bound
in \eqref{eq:kl-sbon} is at least as tight. 
\end{corollary}

\noindent
The proof is given in Appendix~\ref{proof_cor:sobn}. The corollary has three
main implications for test-time alignment, which we outline below.

\paragraph{KL Contractivity.}
The finite-$N$ SBoN policy is always closer to the base policy $\pi_0$ than the
population Gibbs policy $\pi_\beta$ is to $\pi_0$.
Thus, although $\pi_{N,\beta} \to \pi_\beta$ as $N$ increases, its
finite-$N$ distributional shift cannot exceed the population shift induced
by the chosen alignment strength $\beta$.

\paragraph{Bounded Overshooting.}
The expected estimator $\E[\widehat D_N]$ provides an easy-to-estimate  upper bound on the finite-$N$ policy divergence. In particular, for fixed $\beta$ and $N$,  we can sample
candidate pools from $\pi_0$ offline and average $\widehat D_N$ to estimate this quantity. By \Cref{prop:kl_estimator},
\begin{align}
&D_{\mathrm{KL}}\!\left(
\pi_{N,\beta}(\cdot\mid x)
\middle\|
\pi_0(\cdot\mid x)
\right) \notag\\
&\qquad\leq \E[\widehat D_N]
\leq
D_{\mathrm{KL}}\!\left(
\pi_\beta(\cdot\mid x)
\middle\|
\pi_0(\cdot\mid x)
\right)+\log 2.\notag
\end{align}
Thus, the estimator-based upper bound cannot overshoot the theoretical limit given by the population divergence $D_{\mathrm{KL}}(\pi_\beta(\cdot\mid x)\|
\pi_0(\cdot\mid x))$  by more than the universal additive constant $\log 2$. This contrasts with the
bounded-reward guarantee of \citet{aminian2026best}, which can substantially exceed the population divergence and therefore become loose as an upper bound on the finite-$N$ policy shift. In particular, their bound approaches $\beta R_{\max}$ as $N$ grows, whereas our bound never exceeds the population divergence by more than $\log 2$, illustrated in \Cref{fig:softbon}. %

\paragraph{Tightness and Tracking.}
When $0\le r(x,y)\le R_{\max}$, the corollary further shows that $\mathbb{E}[\widehat D_N]$ is no larger than the existing bounded-reward guarantee in~\eqref{bound:r_max}. Combined with the universal $\log 2$ overshoot guarantee above, this comparison is particularly meaningful. An upper bound could improve only slightly upon the bounded-reward guarantee while still remaining far from the true divergence. In contrast, $\mathbb{E}[\widehat D_N]$ is guaranteed to remain within $\log 2$ of the population divergence from above. Thus, the estimator $\E[\widehat D_N]$ closely tracks the true divergence within its fluctuation range. Moreover, our general result does not require bounded rewards and applies whenever $\pi_\beta$ is well-defined and
\(
D_{\mathrm{KL}}\!\left(
\pi_\beta(\cdot\mid x)
\middle\|
\pi_0(\cdot\mid x)
\right)<\infty.
\)
\\

Finally, we note that our results for test-time alignment are not restricted to the LLM setting. The same analysis applies more generally to continuous generative models, including diffusion models, whenever test-time alignment is performed through importance-weighted sampling or resampling.

\input{plots/combined_sBoN_exp} %
\subsection{Experiments}\label{sbon_exp}
\paragraph{Bound comparison. } We consider a discrete output space with eight outcomes and implicitly
condition throughout on a fixed input $x$. The small output space allows us
to compute the finite-$N$ SBoN policy $\pi_{N,\beta}$ exactly, providing
ground-truth finite-$N$ KL values against which to evaluate our bounds.
We fix the base policy and construct four reward profiles ranging from
aligned to anti-aligned with respect to $\pi_0$, as shown in the first row
of \Cref{fig:softbon-alignment}. In the aligned case, high reward is assigned to outputs that already have large probability under the base policy; in the anti-aligned case, high reward is assigned to rare outputs. The two intermediate reward profiles interpolate between these extremes. All rewards take values in $[0,1]$, so the bounded-reward guarantee is identical across reward geometries for fixed $(N,\beta)$.

For each reward profile, we evaluate SBoN over several values of the inverse
temperature $\beta$ and computational budget $N$. Because the output space
is small, we can compute the finite-$N$ SBoN policy $\pi_{N,\beta}$ exactly by
enumerating candidate-count configurations and hence obtain
$D_{\mathrm{KL}}(\pi_{N,\beta}\Vert\pi_0)$ without Monte Carlo error.
We also compute the population divergence
$D_{\mathrm{KL}}(\pi_\beta\Vert\pi_0)$ exactly. The expectation
$\mathbb{E}[\widehat D_N]$ is estimated using $50{,}000$ independent
candidate pools, with Monte Carlo uncertainty reported in the figure. This setup isolates the effect of reward--reference alignment: the base
policy, reward range, and inference parameters are held fixed across reward
geometries, while only the locations receiving high reward are changed.

Figure~\ref{fig:softbon-alignment} illustrates how the SBoN policy changes with both the alignment strength $\beta$ and the reward--reference geometry. As $\beta$ increases, the Gibbs policy places increasingly more mass on high-reward outputs, resulting in larger population and finite-$N$ divergences from the base policy. The magnitude of this shift depends strongly on the reward geometry: when high-reward outputs are already likely under $\pi_0$, the population KL remains relatively small, whereas rewarding outputs that are rare under $\pi_0$ produces a substantially larger population KL. Across all of these regimes, $\mathbb{E}[\widehat D_N]$ closely tracks the exact finite-$N$ divergence $D_{\mathrm{KL}}(\pi_{N,\beta}\Vert\pi_0)$, while the population divergence $D_{\mathrm{KL}}(\pi_\beta\Vert\pi_0)$ characterizes its limiting value as $N\to\infty$. Thus, $D_{\mathrm{KL}}(\pi_\beta\Vert\pi_0)$ provides a theoretical limit on the finite-$N$ divergence, while $\mathbb{E}[\widehat D_N]$ provides a directly estimable upper bound. In contrast, the guarantee of \citet{aminian2026best} is insensitive to the reward--reference geometry and remains unchanged across columns for fixed $(N,\beta)$. 

\paragraph{How many Monte Carlo samples are needed?}
In practice, $\E[\widehat D_N]$ must be estimated from finitely many Monte Carlo
repetitions. We denote by $K$ the Monte Carlo batch size, corresponding to
$K$ independent realizations of $\widehat D_N$, and estimate \begin{talign}
    \E[\widehat D_N] \notag
\approx
\frac{1}{K}\sum_{k=1}^K \widehat D_N^{(k)}.
\end{talign}
How $K$ interacts with $N$ is not immediately clear. For large $N$,
$\widehat D_N$ converges to $D_{\mathrm{KL}}(\pi_{\beta}\|\pi_0)$, suggesting that fewer
repetitions may suffice, whereas for small $N$, the deterministic bound
$0\leq\widehat D_N\leq\log N$ restricts its range. We therefore study how
many repetitions are needed for the empirical average of $\widehat D_N$ to
reliably upper bound $D_{\mathrm{KL}}(\pi_{N,\beta}\|\pi_0)$ across different values of
$N$ and target--proposal divergences.

We consider a discrete state space obtained by discretizing Gaussian
distributions over $1000$ equally spaced points. We vary the
target--proposal mismatch to obtain different values of
$D_{\mathrm{KL}}(\pi_\beta\|\pi_0)$ and evaluate the estimator across several values
of $N$. To estimate the finite-$N$ divergence
$D_{\mathrm{KL}}(\pi_{N,\beta}\|\pi_0)$, we use the limiting result from
\Cref{prop:snis-lower} with $M=2000$ and $B=50000$. To quantify finite-$K$
uncertainty, we repeatedly compute the Monte Carlo average above and report
its $2.5$th and $97.5$th percentiles as error bars. Wider error bars
 indicate greater variability across independent Monte Carlo
estimates of $\E[\widehat D_N]$.

The results are shown in \Cref{fig:snis-divergence}. As expected, increasing
$K$ reduces the Monte Carlo uncertainty in estimating $\E[\widehat D_N]$,
with the error bars narrowing substantially as $K$ increases. Importantly, even for moderate $K$, the empirical estimates remain close to the true expectation and typically preserve the upper-bound relationship with $D_{\mathrm{KL}}(\pi_{N,\beta}\|\pi_0)$. Although the theoretical range of $\widehat D_N$ grows as
$\log N$, the empirical estimates remain well concentrated, with error
bars that are substantially narrower than the full range of the estimator. We also observe a non-monotonic
dependence on the population KL divergence: the typical range of the error bars increases from
$D_{\rm KL}(\pi_\beta\Vert\pi_0)=1$ to
$D_{\rm KL}(\pi_\beta\Vert\pi_0)=4$, but then decreases as the population
KL divergence is further increased to
$D_{\rm KL}(\pi_\beta\Vert\pi_0)=12$. These behaviors can be heuristically
explained by the variance of the Monte Carlo estimator. In particular, since
\begin{talign*}
0\leq \widehat D_N
\leq C_N
\defeq
\log\left(
\frac{N}{1+(N-1)e^{-\beta R_{\max}}}
\right),
\end{talign*}
we have
\begin{talign*}
\mathrm{Var}(\widehat D_N)
\leq
\E[\widehat D_N]
\left(
C_N-\E[\widehat D_N]
\right).
\end{talign*}
Hence, for an average over $K$ independent replicates, the standard deviation is bounded by
\[
K^{-1/2}
\sqrt{
\E[\widehat D_N]
\left(
C_N-\E[\widehat D_N]
\right)
}.
\]
Thus, the Monte Carlo uncertainty depends not only on the range of
$\widehat D_N$, but also on the location of its expectation within that
range. In particular, for fixed $C_N$, the quantity inside the square root is maximized when $\E[\widehat D_N]=C_N/2$ and decreases as the expectation
approaches either endpoint of $[0,C_N]$, providing a possible
explanation for the observed non-monotonic behavior  above. 

Finally, regarding the number of Monte Carlo samples $K$ required for reliable estimation, $K=16$ appears sufficient to provide a reliable empirical upper bound in the scenarios considered, while $K=64$ tracks the true expectation very closely. Similar values of $K$ exhibit comparable performance across the additional scenarios reported in Appendix~\ref{add_sbon_exp}.

\section{Relative Entropy Coding}
\label{sec:relative-entropy-coding}

Relative entropy coding (REC) considers the problem of communicating a sample from a target distribution using shared randomness from a proposal, with coding cost closely related to the KL divergence between the two distributions ~\citep{theis2022algorithms}. REC has recently gained interest in machine learning through the idea of requential coding, where it is used to quantify the information required to transfer knowledge from a teacher model to a student model~\citep{qiu2026requential}. 

Communication-efficient importance sampling-based compression (CE-ISC)~\citep{phan2024importance} implements REC using SNIS by generating shared proposal samples and communicating the index of the selected sample. However, their existing finite-sample coding cost guarantee relies on restrictive assumptions and can be highly conservative for small to moderate $N$. In this section, we show that our general SNIS result yields an immediate finite-$N$ coding guarantee with a universal additive penalty for CE-ISC. Throughout this section, all logarithms have base $2$, as the coding costs are measured in bits.

\paragraph{Setup.}
Let $(X,Y)\sim p_{X,Y}$. For a realization $X=x$, the encoder and decoder
share $N$ i.i.d.\ proposals
\[
Y_1,\ldots,Y_N\overset{\mathrm{i.i.d.}}{\sim}p_Y.
\]
For $i=1,\ldots,N$, we define
\[
w_x(y)\defeq \frac{p_{Y\mid X}(y\mid x)}{p_Y(y)}
\qtext{and}
\lambda_{i,x}\defeq 
\frac{w_x(Y_i)}{\sum_{j=1}^N w_x(Y_j)},
\]
and  denote the corresponding SNIS divergence estimator by
\[
\widehat D_{N,x}
\defeq 
D_{\mathrm{KL}}(\bm{\lambda}_x\|U_N)
=
\sum_{i=1}^N \lambda_{i,x}\log_2(N\lambda_{i,x}).
\]
CE-ISC~\citep{phan2024importance} selects one of the proposals according to
these normalized weights and communicates a random index $J$ identifying the
selected proposal. Conditioned on $X=x$ and the proposal pool,
\citet{phan2024importance} show that
\begin{align}
&\E\!\left[
\log_2 J \mid X=x,Y_{1:N}
\right] \notag\\
&\qquad\leq
\widehat D_{N,x}+\delta,
\qquad
\delta\defeq 1+\frac{\log_2 e}{e}.
\end{align}
Hence,
\[
\E[\log_2 J\mid X=x]
\le
\E[\widehat D_{N,x}\mid X=x]+\delta.
\]

\noindent Their analysis then establishes a finite-sample bound of the form %
\[
\E[\widehat D_{N,x}\mid X=x]
\le
D_{\mathrm{KL}}\!\left(
p_{Y\mid X}(\cdot\mid x)\|p_Y
\right)
+\frac{\Delta_x}{N},
\]
where the remainder term $\Delta_x/N$ depends on bounded likelihood ratios and  higher-order moments of the importance weights; see
\citet[Section~8.2]{phan2024importance} for the explicit expression for
$\Delta_x$. Averaging over $X$ yields %
\[
\E_X\!\left[
\E[\widehat D_{N,X}\mid X]
\right]
\le
I(X;Y)+\frac{\Delta}{N},
\]
for $\Delta\defeq\E[\Delta_X]$, which  leads to their entropy bound
\[
H(J)
\le
I(X;Y)+\frac{\Delta}{N}
+
\log_2\!\left(
I(X;Y)+\frac{\Delta}{N}+1
\right)
+4.
\label{phan_bound}
\]
Although the $\Delta/N$ term vanishes as $N\to\infty$, the finite-sample guarantee relies on strong assumptions. In particular, $\Delta$ contains a term $6(\omega-1)\log\omega$, where
\[
\omega
=
\sup_{x,y}
\frac{p_{Y\mid X}(y\mid x)}{p_Y(y)},
\]
in addition to terms involving higher-order moments. Consequently, the bound can become uninformative for small to moderate $N$ when the likelihood ratio or the corresponding moments are large. Our moment-free SNIS bound gives the following alternative.

\begin{corollary}[Finite-sample coding bound]
\label{cor:ce-isc}
Under the CE-ISC construction,
\[
H(J)
\le
I(X;Y)+\log_2(I(X;Y)+2)+5.
\]
\end{corollary}

\begin{proof}
Applying \Cref{prop:kl_estimator} conditionally on $X=x$:
\[
\E[\widehat D_{N,x}\mid X=x]
\le
D_{\mathrm{KL}}\!\left(
p_{Y\mid X}(\cdot\mid x)\|p_Y
\right)+1.
\]
Averaging over $X$ therefore yields
\[
\E_X\!\big[
\E[\widehat D_{N,X}\mid X]
\big]
\le I(X;Y)+1.
\]
Replacing the bound
$\E[\widehat D_{N,X}]\le I(X;Y)+\Delta/N$
in the coding argument of \citet{phan2024importance} above yields
\[
H(J)\le I(X;Y)+\log_2(I(X;Y)+2)+5,
\]
which concludes the proof.
\end{proof}

\noindent Thus, our result provides a universal finite-sample coding
guarantee  without requiring bounded
likelihood ratios or higher-order moment assumptions.

\section{Related Work}
We study the finite-sample behavior and applications of the estimator $\widehat D_N$. This estimator has previously  appeared in the adaptive importance sampling and sequential Monte Carlo literature. In particular, ~\cite{cornebise2008adaptive} proposed using this term to estimate the Kullback--Leibler divergence between target and proposal distributions and employed this criterion to adapt proposal distributions in particle filtering. \citet{cappe2008adaptive} use the entropy $H(\lambda)$ of the normalized importance weights to monitor convergence in adaptive importance sampling with mixture proposals. It also appears in \citet{brekelmans2023improving} in the context of importance-weighted variational inference analysis. Our work studies this same finite-sample quantity but focuses on a nonasymptotic characterization of its expectation. Previous works \citep{agapiou2017importance, deligiannidis2024importance} developed a general framework for finite-sample analysis of self-normalized importance sampling, relying on certain moment assumptions. Building on this framework, \citet{phan2024importance} derived a finite-sample bound for $\E[\widehat D_N]$, but their result relies on several restrictive boundedness assumptions. Closely related, ~\cite{chatterjee2018sample} show that, under broad conditions, the characteristic sample size required for successful importance sampling scales as $N \asymp \exp\!\left(D_{\mathrm{KL}}(p\|q)\right),$ highlighting the forward KL divergence as a fundamental measure of target--proposal mismatch. \citet{sanz2018importance} further derived nonasymptotic necessary sample-size conditions from an information-theoretic perspective, involving the KL  and other $f$-divergences.

Recent work has also studied importance-weighted resampling in the context of test-time alignment. \cite{verdun2025soft} introduced Soft Best-of-$N$ sampling and analyzed the discrepancy between its finite-sample distribution $\pi_{N,\beta}$ and the tilted target $\pi_\beta$. They establish nonasymptotic bounds on $D_{\mathrm{KL}}(\pi_\beta(\cdot \mid x) \|\pi_{N,\beta}(\cdot \mid x))$ and show that it decays at rate $O(1/N)$ under their setting. Their analysis primarily characterizes the convergence of the finite-sample SBoN policy to the population target, whereas our focus is on providing finite-sample certificates for the distributional shift from the base policy. In particular, our results complement this analysis by providing distribution-free control of KL-related quantities $D_{\mathrm{KL}}(\pi_{N,\beta}(\cdot \mid x)\|\pi_0(\cdot \mid x))$  arising from self-normalized importance weights without requiring bounded rewards or higher-moment assumptions. Other works \citep{beirami2024theoretical, mroueh2024information} study and provide theoretical analysis for the related problem of Best-of-$N$, which arises as a special case of our framework. Finally, \citet{aminian2026best} analyze the same quantity and derive a bounded-reward bound, for which we establish a tighter guarantee.

\section{Conclusion}
This work provides nonasymptotic guarantees for a self-normalized importance sampling estimator of the forward KL divergence. Our results characterize its finite-sample behavior without requiring bounded importance weights or finite higher moments and establish a connection between the normalized importance weights and the discrepancy between the target and proposal distributions. We further demonstrate the applicability of these guarantees to test-time alignment and relative entropy coding. 

Several directions remain for future work. On the theoretical side, our results raise the broader question of what other universal, distribution-free finite-sample guarantees can be established for SNIS, including for other functionals of the normalized importance weights, and whether the additive slack in our bounds can be further tightened. On the algorithmic side, these guarantees may inform more efficient sampling strategies for adaptive importance sampling, where the proposal or sample design is chosen to reduce target--proposal mismatch. Given that recent works have also explored new approaches for scaling test-time compute to improve performance and inference efficiency \citep{lifshitz2025multi,geuter2026guided,di2026best}, extending our analysis to these settings could provide principled control of the resulting distributional shift as well as theoretical guidance for efficient test-time computation.

\bibliographystyle{plainnat}
\bibliography{main}

\clearpage
\appendix

\thispagestyle{empty}

\onecolumn
\aistatstitle{
Supplementary Materials}

\section{Asymptotic Convergence}
\label{sec:asymptotic}

We first recall the Strong Law of Large Numbers and the Dominated
Convergence Theorem, the two main tools we use to establish the convergence results in this section.

\begin{lemma}[Strong Law of Large Numbers; {\citet[Section~14.5]{williams1991probability}}]
Let $Z_1,Z_2,\ldots$ be i.i.d.\ random variables such that
\[
\E[|Z_i|] < \infty.
\]
Then
\(
\frac{1}{N}\sum_{i=1}^N Z_i
\)
converges to $\E[Z_i]$ both almost surely and in $L^1$ as $N\to\infty$.
\end{lemma}

\begin{lemma}[Dominated Convergence Theorem; {\citet[Theorem~1.6.7]{durrett2019probability}}]
Let $Z_1,Z_2,\ldots$ be random variables such that
$Z_N \to Z$ almost surely. Suppose there exists a  random
variable $Y$ such that $\E[Y] < \infty$ and
\[
|Z_N| \leq Y
\qquad \text{ for all } N.
\]
Then 
\[
\lim_{N\to\infty} \E[Z_N] = \E[Z].
\]
\end{lemma}

We now establish the asymptotic convergence of
$\widehat D_N$ and $\E[\widehat D_N]$ to $D_{\mathrm{KL}}(p\|q)$.
\begin{proposition}[Asymptotic consistency]\label{prop:consistency}
Suppose that $p$ and $q$ have the same support and that
$D_{\mathrm{KL}}(p\|q)<\infty$. Then
$$
\widehat D_N
\xrightarrow[N\to\infty]{\mathrm{a.s.}}
D_{\mathrm{KL}}(p\|q).
$$
Furthermore:
\[
\lim_{N\to\infty}\E[\widehat D_N]
=
D_{\mathrm{KL}}(p\|q).
\]

\end{proposition}

\begin{proof}
 Let

$$
W_i\defeq \frac{p(X_i)}{q(X_i)},
\qquad
\overline W_N\defeq \frac{1}{N}\sum_{i=1}^N W_i.
$$
\noindent \textbf{First Claim. } Since \(\lambda_i=W_i/(N\overline W_N)\), we can write
\begin{align*}
\widehat D_N
&=
\sum_{i=1}^N
\frac{W_i}{N\overline W_N}
\log\left(\frac{W_i}{\overline W_N}\right)\\
&=
\frac{1}{\overline W_N}
\left(
\frac{1}{N}\sum_{i=1}^N W_i\log W_i
\right)
-\log\overline W_N.
\end{align*}

\noindent By definition of the likelihood ratio,

$$
\mathbb E[W_i]=1,
\qquad
\mathbb E[W_i\log W_i]
=
D_{\mathrm{KL}}(p\|q) < \infty \quad \text{(by assumption)}.
$$

\noindent Therefore, by the strong law of large numbers %
$$
\overline W_N\xrightarrow{\mathrm{a.s.}}1,
\qquad
\frac{1}{N}\sum_{i=1}^N W_i\log W_i
\xrightarrow{\mathrm{a.s.}}
D_{\mathrm{KL}}(p\|q).
$$
Consequently,
$$
\log\overline W_N\xrightarrow{\mathrm{a.s.}}0,
$$
and substituting these limits into the expression above yields
$$
\widehat D_N
\xrightarrow{\mathrm{a.s.}}
D_{\mathrm{KL}}(p\|q),
$$
which completes the proof for the first part.

\noindent \textbf{Second Claim. }We now prove the second claim,
\(
\E[\widehat D_N]\longrightarrow D_{\mathrm{KL}}(p\|q).
\)
Although we have already shown that $\widehat D_N\to D_{\mathrm{KL}}(p\|q)$ almost surely, almost-sure convergence alone does not justify interchanging the limit and expectation. We will therefore prove convergence of the expectations by showing separately that the positive and negative parts of an equivalent representation of $\widehat D_N$ converge in expectation.

\noindent
\textit{Proof strategy.}
By exchangeability, we will first rewrite
\[
\E[\widehat D_N]=\E[A_N\log A_N],
\qquad
A_N\defeq
\frac{N W_1}{W_1+\sum_{j=2}^N W_j}
=
\frac{1}{
\frac{1}{N}
+
\frac{1}{N W_1}\sum_{j=2}^N W_j
}.
\]
By the strong law of large numbers,
\[
A_N
\xrightarrow{\mathrm{a.s.}}
\frac{1}{1/W_1}
=
W_1.
\] 
Since $0<A_N\leq N$, we have
\[
(A_N\log A_N)^+ \leq N\log N,
\]
while, since $t\log t\geq -1/e$ for all $t>0$,
\[
(A_N\log A_N)^- \leq \frac{1}{e}.
\]
Thus both the positive and negative parts of $A_N\log A_N$ have finite
expectation for every $N$, and hence
\[
\E[\widehat D_N]
=
\E[A_N\log A_N]
=
\E[(A_N\log A_N)^+]
-
\E[(A_N\log A_N)^-].
\]
Similarly, since $D_{\mathrm{KL}}(p\|q)
=\E[W_1\log W_1]$ exists and is finite, $W_1\log W_1$ is integrable,
and therefore
\[
D_{\mathrm{KL}}(p\|q)
=
\E[(W_1\log W_1)^+]
-
\E[(W_1\log W_1)^-].
\]
It therefore suffices to prove
\[
\E[(A_N\log A_N)^+]
\longrightarrow
\E[(W_1\log W_1)^+]
\]
and
\[
\E[(A_N\log A_N)^-]
\longrightarrow
\E[(W_1\log W_1)^-].
\]
Indeed, subtracting the two convergent sequences then yields
\[
\E[\widehat D_N]
\longrightarrow
\E[W_1\log W_1]
=
D_{\mathrm{KL}}(p\|q).
\]
For the negative part, we will apply the Dominated Convergence Theorem.
Since $A_N\to W_1$ almost surely and the map
\(
x\mapsto (x\log x)^-
\)
is continuous on $[0,\infty)$, with the convention $0\log 0=0$, we have
\[
(A_N\log A_N)^-
\xrightarrow{\mathrm{a.s.}}
(W_1\log W_1)^-.
\]
Moreover, the function $x\mapsto x\log x$ is bounded below by $-1/e$
on $[0,\infty)$, and hence
\[
0\leq (A_N\log A_N)^-\leq \frac{1}{e}
\qquad\text{for every }N.
\]
Thus the sequence is dominated by the integrable constant $1/e$.
The dominated convergence theorem therefore yields
\[
\E[(A_N\log A_N)^-]
\longrightarrow
\E[(W_1\log W_1)^-].
\]

For the positive part, we now prove that
\[
\E[(A_N\log A_N)^+]
\longrightarrow
\E[(W_1\log W_1)^+].
\]
Unlike the negative part, $(A_N\log A_N)^+$ is not uniformly bounded.
We therefore split its expectation into  two events. In particular, define
\[
G_N
\defeq
\left\{
\sum_{i=2}^N W_i\geq\frac{N-1}{2}
\right\}.
\]
Then
\begin{align*}
\E[(A_N\log A_N)^+]
&=
\E\left[(A_N\log A_N)^+\mathbf{1}_{G_N}\right] +
\E\left[(A_N\log A_N)^+\mathbf{1}_{G_N^c}\right].
\end{align*}
We treat these two terms separately.

\medskip
\noindent
\textit{First Term.}
\noindent
On the good event $G_N$, we apply the dominated convergence theorem.
By the strong law of large numbers,
\[
\frac{1}{N-1}\sum_{i=2}^N W_i
\xrightarrow{\mathrm{a.s.}}1,
\]
so $\mathbf{1}_{G_N}\to1$ almost surely. Since $A_N\to W_1$
almost surely, it follows that
\[
(A_N\log A_N)^+\mathbf{1}_{G_N}
\xrightarrow{\mathrm{a.s.}}
(W_1\log W_1)^+.
\]

Moreover, on $G_N$,
\[
A_N
=
\frac{N W_1}{W_1+\sum_{i=2}^N W_i}
\leq
\frac{2N}{N-1}W_1
\leq 4W_1,
\]
for $N\geq2$. Hence, on $G_N$,
\[
(A_N\log A_N)^+
\leq
(4W_1\log(4W_1))^+.
\]
Moreover,
\[
(4W_1\log(4W_1))^+
\leq
4(\log 4)W_1+4(W_1\log W_1)^+,
\]
whose expectation is finite since $\E[W_1]=1$ and
$\E[(W_1\log W_1)^+]<\infty$, as shown above. Therefore, by the
dominated convergence theorem,
\[
\E\!\left[
(A_N\log A_N)^+\mathbf{1}_{G_N}
\right]
\longrightarrow
\E[(W_1\log W_1)^+].
\]
\medskip
\noindent
\textit{Second Term.}
It remains to prove that
\[
\E\left[
(A_N\log A_N)^+\mathbf{1}_{G_N^c}
\right]
\longrightarrow0.
\]
We will first use the exponential Markov inequality to show that $\Pr(G_N^c)$ decays geometrically in $N$, and then combine this with
a deterministic upper bound on $(A_N\log A_N)^+$.

Set
\[
T_N\defeq\sum_{i=2}^N W_i.
\]
Then
\[
G_N^c
=
\left\{
T_N<\frac{N-1}{2}
\right\}.
\]
Fix $\theta>0$. Since the map $x\mapsto e^{-\theta x}$ is decreasing,
\[
\left\{
T_N<\frac{N-1}{2}
\right\}
=
\left\{
e^{-\theta T_N}>
e^{-\theta(N-1)/2}
\right\}.
\]
Applying Markov's inequality to the nonnegative random variable
$e^{-\theta T_N} \leq 1$ gives
\begin{align*}
\Pr(G_N^c)
&\leq
\frac{\E[e^{-\theta T_N}]}
     {e^{-\theta(N-1)/2}}\\
&=
e^{\theta(N-1)/2}\E[e^{-\theta T_N}].
\end{align*}
Since $W_2,\ldots,W_N$ are independent and identically distributed,
\[
\E[e^{-\theta T_N}]
=
\prod_{i=2}^N\E[e^{-\theta W_i}]
=
\left(\E[e^{-\theta W_1}]\right)^{N-1}.
\]
Therefore
\[
\Pr(G_N^c)
\leq
\left(
e^{\theta/2}\E[e^{-\theta W_1}]
\right)^{N-1}.
\]

We now show that $\theta>0$ can be chosen so that the quantity in
parentheses is strictly smaller than one. By \citet[Chapter~8.2.4]{le2022measure}, for a random variable $Z\geq0$ and
$\E[Z]<\infty$,  its Laplace transform
\[
L_Z(\theta)\defeq \E[e^{-\theta Z}]
\]
satisfies
\[
L_Z'(0+)=-\E[Z].
\]
Applying this with $Z=W_1$ and using $\E[W_1]=1$, we obtain
\[
\lim_{\theta\downarrow0}
\frac{\E[e^{-\theta W_1}]-1}{\theta}
=-1.
\]
Equivalently, defining
\[
r(\theta)
\defeq
\E[e^{-\theta W_1}]-(1-\theta),
\]
we have
\[
\frac{r(\theta)}{\theta}
=
\frac{\E[e^{-\theta W_1}]-1}{\theta}+1
\longrightarrow0.
\]
Thus $r(\theta)=o(\theta)$ as $\theta\downarrow0$; that is, the
remainder $r(\theta)$ is negligible compared with $\theta$ as
$\theta\downarrow0$. Hence
\[
\E[e^{-\theta W_1}]
=
1-\theta+o(\theta).
\]
Moreover, by Taylor's theorem,
\[
e^{\theta/2}
=
1+\frac{\theta}{2}+o(\theta).
\]
Combining the two expansions gives
\begin{align*}
e^{\theta/2}\E[e^{-\theta W_1}]
&=
\left(1+\frac{\theta}{2}+o(\theta)\right)
\left(1-\theta+o(\theta)\right)\\
&=
1-\frac{\theta}{2}+o(\theta).
\end{align*}
In particular, for all sufficiently small $\theta>0$,
\[
e^{\theta/2}\E[e^{-\theta W_1}]<1.
\]
Fix such a $\theta$ and define
\[
\rho
\defeq
e^{\theta/2}\E[e^{-\theta W_1}],
\qquad 0<\rho<1.
\]
Therefore,
\[
\Pr(G_N^c)\leq \rho^{N-1}.
\]

It remains to control the size of the integrand on $G_N^c$. Since
\[
A_N
=
\frac{N W_1}{\sum_{j=1}^N W_j}
\leq N,
\]
we have
\[
(A_N\log A_N)^+
\leq N\log N.
\]
Therefore
\begin{align*}
0
&\leq
\E\left[
(A_N\log A_N)^+\mathbf{1}_{G_N^c}
\right]\\
&\leq
N\log N\,\Pr(G_N^c)\\
&\leq
N\log N\,\rho^{N-1}.
\end{align*}
Since $0<\rho<1$, let $c\defeq-\log\rho>0$. Then
\[
N\log N\,\rho^{N-1}
=
e^c N\log N\,e^{-cN}.
\]
By two applications of L'H\^opital's rule,
\[
\lim_{z\to\infty}\frac{z\log z}{e^{cz}}
=
\lim_{z\to\infty}\frac{\log z+1}{ce^{cz}}
=
\lim_{z\to\infty}\frac{1/z}{c^2e^{cz}}
=
0.
\]
Therefore,
\[
N\log N\,\rho^{N-1}\longrightarrow0.
\]
Thus
\[
\E\left[
(A_N\log A_N)^+\mathbf{1}_{G_N^c}
\right]
\longrightarrow0.
\]

Combining both contributions, we conclude that
\[
\E[(A_N\log A_N)^+]
\longrightarrow
\E[(W_1\log W_1)^+],
\]
which completes the proof.
\end{proof}

\section{A Contraction Lemma }\label{supp_lemma}
We first provide the following result  which will be used in the proof of this part. 
\begin{lemma}[Contraction Lemma]\label{lemma_mean}
Let $T\geq0$ be an integrable random variable, and let
$g:(0,\infty)\to(0,\infty)$. Suppose there exists $a>0$ such that
\[
\min\{t,a\}
\le
g(t)
\le
\max\{t,a\}
\qquad
\text{for all } t>0,
\]
and
\[
\mathbb{E}[g(T)]
=
\mathbb{E}[T].
\]
Then, for every convex differentiable function $f$ for which the
expectations are finite,
\[
\mathbb{E}[f(g(T))]
\le
\mathbb{E}[f(T)].
\]
In particular, for $f(t)=t\log t$,
\[
\mathbb{E}[g(T)\log g(T)]
\le
\mathbb{E}[T\log T].
\]
\end{lemma}

\begin{proof}
Let
\[
\ell_a(t)
=
f(a)+f'(a)(t-a)
\]
be the tangent line to $f$ at $a$, and define
\[
R_a(t)
=
f(t)-\ell_a(t).
\]
Since $f$ is convex and $\ell_a$ is linear, we have $R_a$ is also convex. Furthermore,
\[
R_a'(t)
=
f'(t)-f'(a),
\]
thus $R_a(t)$ achieves minimal value at $a$. For this reason, we have $R_a$ is nonincreasing on $(0,a]$ and nondecreasing on
$[a,\infty)$. Because $g(t)$ lies between $t$ and $a$, i.e.:
\[
\min\{t,a\}
\le
g(t)
\le
\max\{t,a\}
\qquad
\text{for all } t \geq0,
\]

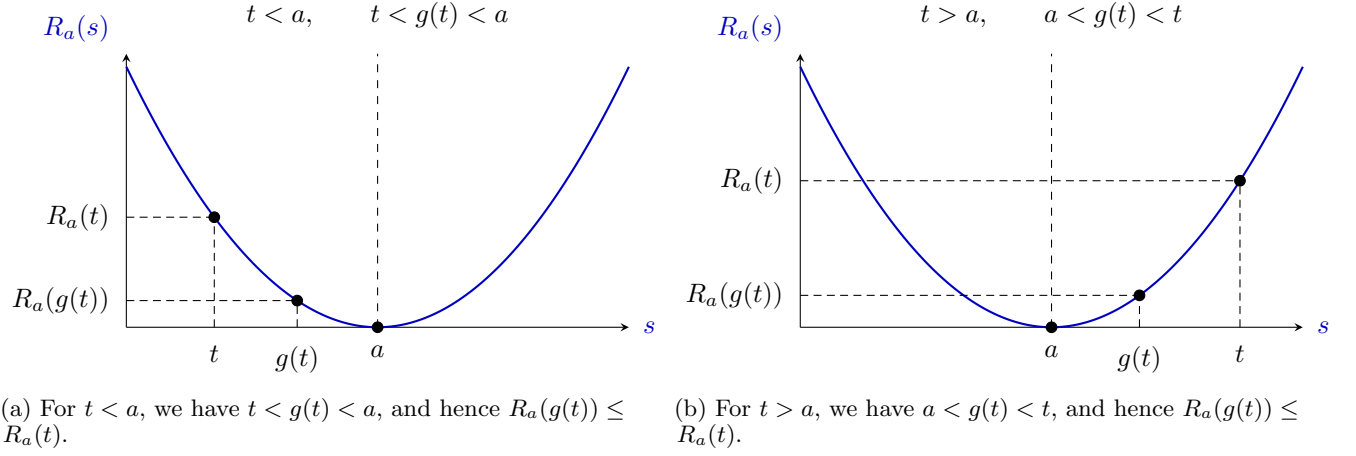
\begin{figure}[t]
    \centering

    \begin{subfigure}[t]{0.48\textwidth}
        \centering
        \begin{tikzpicture}
        \begin{axis}[
            width=\linewidth,
            height=5.2cm,
            axis lines=left,
            xlabel={},
            ylabel={},
            xmin=0, xmax=2,
            ymin=0, ymax=1.05,
            xtick=\empty,
            ytick=\empty,
            samples=200,
            domain=0:2,
            grid=none,
            title={$t<a,\qquad t<g(t)<a$},
            clip=false,
        ]

        \addplot[
            blue!75!black,
            thick
        ] {(x-1)^2};

        \addplot[dashed] coordinates {
            (1,0)
            (1,1.05)
        };

        \addplot[densely dashed, thin] coordinates {
            (0,0.4225)
            (0.35,0.4225)
        };

        \addplot[densely dashed, thin] coordinates {
            (0,0.1024)
            (0.68,0.1024)
        };

        \addplot[densely dashed, thin] coordinates {
            (0.35,0)
            (0.35,0.4225)
        };

        \addplot[densely dashed, thin] coordinates {
            (0.68,0)
            (0.68,0.1024)
        };

        \addplot[only marks, mark=*] coordinates {
            (0.35,0.4225)
            (0.68,0.1024)
            (1,0)
        };

        \node[anchor=north] at (axis cs:0.35,-0.035) {$t$};
        \node[anchor=north] at (axis cs:0.68,-0.035) {$g(t)$};
        \node[anchor=north] at (axis cs:1,-0.035) {$a$};

        \node[anchor=east] at (axis cs:-0.03,0.4225)
            {$R_a(t)$};

        \node[anchor=east] at (axis cs:-0.03,0.1024)
            {$R_a(g(t))$};

        \node[
            anchor=west,
            text=blue!75!black
        ] at (axis cs:2.02,0) {$s$};

        \node[
            anchor=south east,
            text=blue!75!black
        ] at (axis cs:-0.03,1.06) {$R_a(s)$};

        \end{axis}
        \end{tikzpicture}

        \caption{
        For $t<a$, we have $t<g(t)<a$, and hence
        $R_a(g(t))\le R_a(t)$.
        }
    \end{subfigure}
    \hfill
    \begin{subfigure}[t]{0.48\textwidth}
        \centering
        \begin{tikzpicture}
        \begin{axis}[
            width=\linewidth,
            height=5.2cm,
            axis lines=left,
            xlabel={},
            ylabel={},
            xmin=0, xmax=2,
            ymin=0, ymax=1.05,
            xtick=\empty,
            ytick=\empty,
            samples=200,
            domain=0:2,
            grid=none,
            title={$t>a,\qquad a<g(t)<t$},
            clip=false,
        ]

        \addplot[
            blue!75!black,
            thick
        ] {(x-1)^2};

        \addplot[dashed] coordinates {
            (1,0)
            (1,1.05)
        };

        \addplot[densely dashed, thin] coordinates {
            (0,0.1225)
            (1.35,0.1225)
        };

        \addplot[densely dashed, thin] coordinates {
            (0,0.5625)
            (1.75,0.5625)
        };

        \addplot[densely dashed, thin] coordinates {
            (1.35,0)
            (1.35,0.1225)
        };

        \addplot[densely dashed, thin] coordinates {
            (1.75,0)
            (1.75,0.5625)
        };

        \addplot[only marks, mark=*] coordinates {
            (1,0)
            (1.35,0.1225)
            (1.75,0.5625)
        };

        \node[anchor=north] at (axis cs:1,-0.035) {$a$};
        \node[anchor=north] at (axis cs:1.35,-0.035) {$g(t)$};
        \node[anchor=north] at (axis cs:1.75,-0.035) {$t$};

        \node[anchor=east] at (axis cs:-0.03,0.1225)
            {$R_a(g(t))$};

        \node[anchor=east] at (axis cs:-0.03,0.5625)
            {$R_a(t)$};

        \node[
            anchor=west,
            text=blue!75!black
        ] at (axis cs:2.02,0) {$s$};

        \node[
            anchor=south east,
            text=blue!75!black
        ] at (axis cs:-0.03,1.06) {$R_a(s)$};

        \end{axis}
        \end{tikzpicture}

        \caption{
        For $t>a$, we have $a<g(t)<t$, and hence
        $R_a(g(t))\le R_a(t)$.
        }
    \end{subfigure}

    \caption{
    Illustration of contraction toward the anchor $a$, using
    $R_a(s)=(s-a)^2$.
    }
    \label{Ra_contraction}
\end{figure}
\noindent Therefore we have
\[
R_a(g(t))
\le
R_a(t)
\]
for every $t>0$, illustrated in Figure \ref{Ra_contraction}. Taking expectations gives
\[
\mathbb{E}[f(g(T))-\ell_a(g(T))]
\le
\mathbb{E}[f(T)-\ell_a(T)].
\]
Since $\ell_a$ is affine, we have
\[
\ell_a(t)
=
f(a)+f'(a)(t-a)
=
f'(a)t + f(a)-af'(a).
\]
Therefore,
\[
\begin{aligned}
\mathbb{E}[\ell_a(g(T))]
&=
\mathbb{E}\left[
f(a)+f'(a)(g(T)-a)
\right] \\
&=
f(a)+f'(a)\left(\mathbb{E}[g(T)]-a\right) \\
&=
f(a)+f'(a)\left(\mathbb{E}[T]-a\right) \\
&=
\mathbb{E}\left[
f(a)+f'(a)(T-a)
\right] \\
&=
\mathbb{E}[\ell_a(T)],
\end{aligned}
\]
where the third equality follows from
$\mathbb{E}[g(T)]=\mathbb{E}[T]$.
Consequently,
\[
\mathbb{E}[\ell_a(g(T))]
=
\mathbb{E}[\ell_a(T)].
\]
Thus, from
\[
\mathbb{E}[f(g(T))-\ell_a(g(T))]
\le
\mathbb{E}[f(T)-\ell_a(T)],
\]
the affine terms on both sides are equal and cancel, yielding
\[
\mathbb{E}[f(g(T))]
\le
\mathbb{E}[f(T)],
\] which completes the proof.
\end{proof}

\section{Proof of Theorem \texorpdfstring{\ref{prop:kl_estimator}}{prop:kl\_estimator}}\label{app:b}
Let $p$ and $q$ be probability distributions on $\mathcal X$ with the same
support. Define the likelihood ratio
\[
w(x)=\frac{p(x)}{q(x)}
\]
for $q$-almost every $x\in\mathcal X$. Let
\[
X_1,\ldots,X_N\overset{\mathrm{i.i.d.}}{\sim}q,
\qquad
W_i=w(X_i),
\qquad
S=\sum_{i=1}^N W_i,
\]
where logarithms are natural and $W_i>0$ almost surely. Define
\[
\lambda_i=\frac{W_i}{S}.
\]
Moreover,
\[
\mathbb E_{X_i\sim q}[W_i]
=
\mathbb E_{X_i\sim q}[w(X_i)]
=
\int_{\mathcal X}p(x)\,dx
=
1.
\]

\noindent We prove the inequality
\[
{
\mathbb E_{X_{1:N}\sim q^N}
\left[
N\lambda_1\log(N\lambda_1)
\right]
\le
D_{\mathrm{KL}}(p\|q)+\log 2.
}
\]
By exchangeability, this is equivalent to
\[
\mathbb E_{X_{1:N}\sim q^N}
\left[
\sum_{i=1}^N\lambda_i\log(N\lambda_i)
\right]
\le
D_{\mathrm{KL}}(p\|q)+\log 2.
\tag{1}
\]

\noindent For $N=1$, the result is immediate. Henceforth, assume $N\ge2$.
\subsection{Proof Overview}

\paragraph{Change of law.}
Let
\[
R=q^N
\]
denote the joint law of $(X_1,\ldots,X_N)$, and define
\[
A(X_{1:N})
\defeq 
\frac{N W_1}{S}
=
N\lambda_1,
\qquad
S=\sum_{i=1}^N W_i.
\]
We introduce the tilted law
\[
\widetilde R(X_{1:N})
\defeq 
A(X_{1:N})R(X_{1:N}).
\]

\begin{center}
\fbox{
\begin{minipage}{0.92\linewidth}
\textbf{Claim 1.}
$\widetilde R$ is a probability law.

\medskip

\textit{Proof.}
Clearly, $\widetilde R(X_{1:N})\ge 0$. It remains to verify that
$\widetilde R$ integrates to one.

By exchangeability of $W_1,\ldots,W_N$ under $R=q^N$,
\[
\begin{aligned}
\mathbb E_R[A]
&=
N\mathbb E_R\!\left[\frac{W_1}{S}\right] \\
&=
\mathbb E_R\!\left[
\sum_{i=1}^N\frac{W_i}{S}
\right] \\
&=
\mathbb E_R[1] \\
&=1.
\end{aligned}
\]
Therefore,
\[
\begin{aligned}
\int_{\mathcal X^N}
\widetilde R(x_{1:N})\,dx_{1:N}
&=
\int_{\mathcal X^N}
A(x_{1:N})R(x_{1:N})\,dx_{1:N} \\
&=
\mathbb E_R[A] \\
&=1.
\end{aligned}
\]
Hence, $\widetilde R$ is a valid probability law.
\hfill $\square$
\end{minipage}
}
\end{center}

\noindent We have
\[
\frac{\widetilde R(x_{1:N})}{R(x_{1:N})}
=
A(x_{1:N})
=
\frac{N W_1}{S}
=
N\lambda_1.
\]
Therefore, by the definition of the KL divergence,
\[
\begin{aligned}
D_{\mathrm{KL}}(\widetilde R\|R)
&=
\int_{\mathcal X^N}
\widetilde R(x_{1:N})
\log\!\left(
\frac{\widetilde R(x_{1:N})}
     {R(x_{1:N})}
\right)
\,dx_{1:N} \\
&=
\int_{\mathcal X^N}
A(x_{1:N})R(x_{1:N})
\log A(x_{1:N})
\,dx_{1:N} \\
&=
\mathbb E_R[A\log A] \\
&=
\mathbb E_R\!\left[
\frac{N W_1}{S}
\log\!\left(\frac{N W_1}{S}\right)
\right] \\
&=
\mathbb E_R\!\left[
N\lambda_1\log(N\lambda_1)
\right].
\end{aligned}
\tag{2}
\]

\noindent Therefore, proving
\[
\mathbb E_R\!\left[
\sum_{i=1}^N
\lambda_i\log(N\lambda_i)
\right]
\le
D_{\mathrm{KL}}(p\|q)+\log 2
\]
is equivalent to proving
\[
{
D_{\mathrm{KL}}(\widetilde R\|R)
\le
D_{\mathrm{KL}}(p\|q)+\log 2.
}
\]

\paragraph{Chain Rule for Relative Entropy.}
To establish the desired inequality, separate the first coordinate from the
remaining coordinates. Define
\[
Y=(X_2,\ldots,X_N),
\qquad
T=\sum_{i=2}^N W_i,
\]
so that
\[
S=W_1+T.
\]

\noindent Under the original law
\[
r=q^N=q\otimes q^{N-1},
\]
the random variable \(X_1\) has distribution \(q\), the vector \(Y\) has
distribution \(q^{N-1}\), and \(X_1\) is independent of \(Y\). In particular,
for every fixed value \(y\) of \(Y\),
\[
r_{X_1\mid Y=y}=q.
\]

\noindent Under the tilted law \(\widetilde r\), however, \(X_1\) and \(Y\) are generally
not independent. We denote by
\[
\widetilde r_{X_1\mid Y=y}
\]
the conditional distribution of \(X_1\) when \(Y=y\), and by
\(\widetilde r_Y\) the marginal distribution of \(Y\).

\begin{center}
\fbox{
\begin{minipage}{0.94\linewidth}
\textbf{Claim 2. Divergence decomposition and termwise bounds.}

With
\[
Y=(X_2,\ldots,X_N),
\]
the chain rule for relative entropy gives
\[
D_{\mathrm{KL}}(\widetilde r\|r)
=
\underbrace{
D_{\mathrm{KL}}(\widetilde r_Y\|q^{N-1})
}_{\text{\emph{marginal divergence}}}
+
\underbrace{
\mathbb E_{Y\sim\widetilde r_Y}
\left[
D_{\mathrm{KL}}\!\left(
\widetilde r_{X_1\mid Y}
\,\middle\|\,
q
\right)
\right]
}_{\text{\emph{average conditional divergence}}}.
\tag{3}
\]

We will prove the following two bounds separately:
\[
D_{\mathrm{KL}}(\widetilde r_Y\|q^{N-1})
\le \log 2,
\tag{4}
\]
and
\[
\mathbb E_{Y\sim\widetilde r_Y}
\left[
D_{\mathrm{KL}}\!\left(
\widetilde r_{X_1\mid Y}
\,\middle\|\,
q
\right)
\right]
\le D_{\mathrm{KL}}(p\|q).
\tag{5}
\]

Consequently,
\[
D_{\mathrm{KL}}(\widetilde r\|r)
\le
D_{\mathrm{KL}}(p\|q)+\log 2.
\tag{6}
\]
No additional moment assumptions are required beyond the stated positivity
assumption on $W_i$ and $D_{\mathrm{KL}}(p\|q)<\infty$.

\end{minipage}
}
\end{center}

\noindent Before proving each term in the next sections, we note that the conditional term is
\[
\mathbb E_{Y\sim\widetilde r_Y}
\left[
D_{\mathrm{KL}}\!\left(
\widetilde r_{X_1\mid Y}
\,\middle\|\,
q
\right)
\right]
=
\int
D_{\mathrm{KL}}\!\left(
\widetilde r_{X_1\mid Y=y}
\,\middle\|\,
q
\right)
\widetilde r_Y(y)\,dy.
\]
Thus, the chain rule decomposes the joint divergence into the divergence of
the marginal law of \(Y\) and the average conditional divergence of \(X_1\)
given \(Y=y\).

\subsection{Bounding the Conditional Contribution}

In this section, we prove that
\[
\mathbb E_{Y\sim\widetilde r_Y}
\left[
D_{\mathrm{KL}}\!\left(
\widetilde r_{X_1\mid Y}
\,\middle\|\,
q
\right)
\right]
\le
D_{\mathrm{KL}}(p\|q).
\]

\noindent We establish the stronger pointwise bound
\[
D_{\mathrm{KL}}\!\left(
\widetilde r_{X_1\mid Y=x_{2:N}}
\,\middle\|\,
q
\right)
\le
D_{\mathrm{KL}}(p\|q)
\]
for every \(x_{2:N}\) for which the conditional distribution is well defined.
Once this pointwise inequality is proved, taking expectation with respect to
\(Y\sim\widetilde r_Y\) immediately yields the required quantity since
$D_{\mathrm{KL}}(p\|q)$ is a constant. Before proving the result, we present \Cref{contraction_corollary}, a main tool for proving this result. The proof is
presented at the end of this section.

\begin{center}
\fbox{
\begin{minipage}{0.94\linewidth}

\begin{corollary}[Application of Contraction Lemma]\label{contraction_corollary}
Let \(W>0\) and \(W\sim p_W\) satisfy \(\E[W]=1\). For fixed \(t\ge 0\), define
\[
m_t
=
\E\!\left[\frac{W}{W+t}\right],
\qquad
h_t(w)
=
\frac{w/(w+t)}{m_t}.
\]
Then \(h_t\) is a probability density relative to the law of \(W\), and
\[
\E\!\left[h_t(W)\log h_t(W)\right]
\le
\E[W\log W],
\]
where all expectations are taken with respect to \(W\sim p_W\).
\end{corollary}

\end{minipage}
}
\end{center}

\paragraph{Applying the Corollary.}

For the remainder of this section, fix a realization
\[
Y=x_{2:N},
\qquad
t=\sum_{i=2}^N w(x_i),
\qquad
w(x)=\frac{p(x)}{q(x)}.
\]
Let \(X\sim q\), define
\[
W=w(X),
\]
and denote by \(p_W\) the probability law of \(W\) induced by \(X\sim q\).
We will show that the likelihood ratio of
\(\widetilde r_{X_1\mid Y=x_{2:N}}\) with respect to \(q\) is
\[
\frac{
\widetilde r_{X_1\mid Y=x_{2:N}}(x)
}{
q(x)
}
=
h_t(w(x)),
\]
where
\[
h_t(u)
=
\frac{\dfrac{u}{u+t}}{m_t},
\qquad
m_t
=
\mathbb E_{W\sim p_W}
\left[
\frac{W}{W+t}
\right].
\]

\noindent Once this identity is established, applying directly  \Cref{contraction_corollary} yields the desired pointwise inequality. To see this, we note that the
conditional divergence can be written as
\[
\begin{aligned}
D_{\mathrm{KL}}\!\left(
\widetilde r_{X_1\mid Y=x_{2:N}}
\,\middle\|\,
q
\right)
&=
\int_{\mathcal X}
\widetilde r_{X_1\mid Y=x_{2:N}}(x)
\log
\frac{
\widetilde r_{X_1\mid Y=x_{2:N}}(x)
}{
q(x)
}
\,dx \\
&=
\int_{\mathcal X}
h_t(w(x))\log h_t(w(x))\,q(x)\,dx \\
&=
\mathbb E_{X\sim q}
\left[
h_t(w(X))\log h_t(w(X))
\right].
\end{aligned}
\tag{7}
\]

\noindent By the law of the unconscious statistician, for every measurable
function \(g\),
\[
\mathbb E_{X\sim q}[g(w(X))]
=
\mathbb E_{W\sim p_W}[g(W)].
\]
Applying this identity with
\[
g(u)=h_t(u)\log h_t(u)
\]
gives
\[
\mathbb E_{X\sim q}
\left[
h_t(w(X))\log h_t(w(X))
\right]
=
\mathbb E_{W\sim p_W}
\left[
h_t(W)\log h_t(W)
\right].
\tag{8}
\]

\noindent Before applying the \Cref{contraction_corollary}, we verify that \(W\) satisfies
its normalization assumption. Indeed,
\[
\begin{aligned}
\mathbb E_{W\sim p_W}[W]
&=
\mathbb E_{X\sim q}[w(X)] \\
&=
\int_{\mathcal X}
\frac{p(x)}{q(x)}q(x)\,dx \\
&=
\int_{\mathcal X}p(x)\,dx \\
&=1.
\end{aligned}
\]

\noindent Therefore, the \Cref{contraction_corollary} applies to \(W\sim p_W\) and yields
\[
\mathbb E_{W\sim p_W}
\left[
h_t(W)\log h_t(W)
\right]
\le
\mathbb E_{W\sim p_W}[W\log W].
\tag{9}
\]

\noindent Applying the law of the unconscious statistician once more,
\[
\begin{aligned}
\mathbb E_{W\sim p_W}[W\log W]
&=
\mathbb E_{X\sim q}
\left[
w(X)\log w(X)
\right] \\
&=
\int_{\mathcal X}
w(x)\log w(x)\,q(x)\,dx \\
&=
\int_{\mathcal X}
\frac{p(x)}{q(x)}
\log\!\left(\frac{p(x)}{q(x)}\right)
q(x)\,dx \\
&=
\int_{\mathcal X}
p(x)\log\!\left(\frac{p(x)}{q(x)}\right)\,dx \\
&=
D_{\mathrm{KL}}(p\|q).
\end{aligned}
\tag{10}
\]

\noindent Combining \textup{(7)}--\textup{(10)}, we obtain
\[
D_{\mathrm{KL}}\!\left(
\widetilde r_{X_1\mid Y=x_{2:N}}
\,\middle\|\,
q
\right)
\le
D_{\mathrm{KL}}(p\|q).
\]

\noindent Thus, once the conditional likelihood-ratio identity is verified,
the desired pointwise inequality follows directly from \Cref{contraction_corollary}.
It therefore remains to derive the conditional likelihood-ratio identity
\[
\frac{
\widetilde r_{X_1\mid Y=x_{2:N}}(x)
}{
q(x)
}
=
h_t(w(x)).
\]

\paragraph{Deriving the Conditional Likelihood Ratio.}

Fix a realization
\[
Y=x_{2:N},
\qquad
t=\sum_{i=2}^N w(x_i),
\qquad
w(x)=\frac{p(x)}{q(x)}.
\]
After conditioning on \(Y=x_{2:N}\), the quantity \(t\) is fixed. Recall that,
under the original law \(r=q^N\), \(X_1\) is independent of \(Y\). Therefore,
\[
r_{X_1\mid Y=x_{2:N}}(x)=q(x).
\]

Under the tilted law \(\widetilde r\), the conditional distribution of \(X_1\)
given \(Y=x_{2:N}\) is obtained by normalizing the joint density:
\[
\widetilde r_{X_1\mid Y=x_{2:N}}(x)
=
\frac{
\widetilde r(x,x_{2:N})
}{
\displaystyle
\int_{\mathcal X}
\widetilde r(z,x_{2:N})\,dz
}.
\]
Using
\[
\widetilde r(x_{1:N})
=
A(x_{1:N})r(x_{1:N})
\]
and
\[
r(x,x_{2:N})
=
q(x)\prod_{i=2}^N q(x_i),
\]
we obtain
\[
\begin{aligned}
\widetilde r_{X_1\mid Y=x_{2:N}}(x)
&=
\frac{
A(x,x_{2:N})q(x)
\displaystyle\prod_{i=2}^N q(x_i)
}{
\displaystyle
\int_{\mathcal X}
A(z,x_{2:N})q(z)
\prod_{i=2}^N q(x_i)\,dz
} \\
&=
\frac{
A(x,x_{2:N})q(x)
}{
\displaystyle
\int_{\mathcal X}
A(z,x_{2:N})q(z)\,dz
}.
\end{aligned}
\tag{11}
\]

\noindent For the fixed realization \(x_{2:N}\), we have
\[
A(x,x_{2:N})
=
\frac{Nw(x)}{w(x)+t}.
\]
Substituting this expression into \textup{(11)} gives
\[
\begin{aligned}
\widetilde r_{X_1\mid Y=x_{2:N}}(x)
&=
\frac{
\dfrac{Nw(x)}{w(x)+t}q(x)
}{
\displaystyle
\int_{\mathcal X}
\frac{Nw(z)}{w(z)+t}q(z)\,dz
} \\
&=
\frac{
\dfrac{w(x)}{w(x)+t}
}{
\displaystyle
\int_{\mathcal X}
\frac{w(z)}{w(z)+t}q(z)\,dz
}
q(x).
\end{aligned}
\tag{12}
\]

\noindent Let \(X\sim q\), define
\[
W=w(X),
\]
and let \(p_W\) denote the probability law of \(W\) induced by \(X\sim q\).
By the law of the unconscious statistician,
\[
\begin{aligned}
\int_{\mathcal X}
\frac{w(z)}{w(z)+t}q(z)\,dz
&=
\mathbb E_{X\sim q}
\left[
\frac{w(X)}{w(X)+t}
\right] \\
&=
\mathbb E_{W\sim p_W}
\left[
\frac{W}{W+t}
\right] \\
&=
m_t.
\end{aligned}
\tag{13}
\]

\noindent Combining \textup{(12)} and \textup{(13)}, we obtain
\[
\widetilde r_{X_1\mid Y=x_{2:N}}(x)
=
\frac{\dfrac{w(x)}{w(x)+t}}{m_t}\,q(x)
=
h_t(w(x))q(x).
\]
Consequently, the likelihood ratio of the conditional law with respect to
\(q\) is
\[
\boxed{
\frac{
\widetilde r_{X_1\mid Y=x_{2:N}}(x)
}{
q(x)
}
=
h_t(w(x)).
}
\tag{14}
\]

\subsubsection{Proof of \texorpdfstring{\Cref{contraction_corollary}}{the corollary}}
For convenience, all expectations are taken with respect to $W\sim p_W$.
By construction,
\[
\E[h_t(W)]
=
\frac{1}{m_t}
\E\!\left[\frac{W}{W+t}\right]
=
1.
\]

\paragraph{Inverse Function Analysis.}
Write
\[
c=\frac{1}{m_t},
\qquad
a=c-t,
\qquad
h_t(w)=\frac{cw}{w+t}.
\]

\noindent The key step is to show that \(a>0\), that \(a\) is a fixed point of
\(h_t\), and that, for every \(w>0\), the transformed value \(h_t(w)\)
lies between \(w\) and \(a\). More precisely, we will prove that
\[
h_t(a)=a
\]
and
\[
\min\{w,a\}
\le
h_t(w)
\le
\max\{w,a\},
\qquad w>0.
\tag{15}
\]
\hrule
\begin{proof}

We first verify that \(a>0\). If \(t=0\), then
\[
m_0
=
\mathbb E\!\left[\frac{W}{W}\right]
=
1,
\]
so
\[
c=1,
\qquad
a=c-t=1>0.
\]

\noindent If \(t>0\), then, using \(\mathbb E[W]=1\),
\[
\begin{aligned}
1-tm_t
&=
\mathbb E[W]
-
t\mathbb E\!\left[\frac{W}{W+t}\right] \\
&=
\mathbb E\!\left[
W-\frac{tW}{W+t}
\right] \\
&=
\mathbb E\!\left[
\frac{W^2}{W+t}
\right]
>0.
\end{aligned}
\]
Therefore,
\[
a
=
c-t
=
\frac{1}{m_t}-t
=
\frac{1-tm_t}{m_t}
>0.
\]

\noindent Next, since \(a+t=c\),
\[
h_t(a)
=
\frac{ca}{a+t}
=
\frac{ca}{c}
=
a.
\]
Thus, \(a\) is a fixed point of \(h_t\). Finally,
\[
\begin{aligned}
h_t(w)-w
&=
\frac{cw}{w+t}-w \\
&=
\frac{w(c-w-t)}{w+t} \\
&=
\frac{w(a-w)}{w+t}.
\end{aligned}
\tag{16}
\]
Because \(w/(w+t)>0\), the sign of \(h_t(w)-w\) is the same as the sign
of \(a-w\). Hence,
\[
w\le a
\quad\Longrightarrow\quad
h_t(w)\ge w,
\]
whereas
\[
w\ge a
\quad\Longrightarrow\quad
h_t(w)\le w.
\]

Moreover, \(h_t\) is increasing because
\[
h_t'(w)
=
\frac{ct}{(w+t)^2}
\ge 0.
\]
Since \(h_t(a)=a\), monotonicity gives
\[
w\le a
\quad\Longrightarrow\quad
h_t(w)\le h_t(a)=a,
\]
and
\[
w\ge a
\quad\Longrightarrow\quad
h_t(w)\ge h_t(a)=a.
\]

\noindent Combining these facts proves \textup{(15)}. Thus,
\(h_t(w)\) always lies between \(w\) and the fixed point \(a\).
\end{proof}
\hrule

\vspace{10pt}
\noindent To obtain the final result, we apply Lemma~\ref{lemma_mean} with
\[
g(w)\triangleq h_t(w),
\qquad
f(w)\triangleq w\log w.
\]
The lemma then gives
\[
\mathbb{E}\!\left[
h_t(W)\log h_t(W)
\right]
\le
\mathbb{E}\!\left[
W\log W
\right],
\]
which concludes the proof.

\subsection{Bounding the Marginal Contribution}

We now prove that
\[
D_{\mathrm{KL}}(\widetilde r_Y\|q^{N-1})
\le
\log 2.
\]

\noindent The proof proceeds in three steps. First, we identify the
likelihood ratio of the marginal law \(\widetilde r_Y\) with respect to
\(q^{N-1}\). More precisely, we construct a nonnegative function \(B\)
such that
\[
\frac{\widetilde r_Y(Y)}
{q^{N-1}(Y)}
=
B(Y),
\]
which implies
\[
\mathbb E_{Y\sim q^{N-1}}[B(Y)]
=
1
\qquad \text{, equivalently} \qquad
\int_{\mathcal X^{N-1}}
B(y)q^{N-1}(y)\,dy
=
1.
\]

\noindent Second, we prove the second-moment bound
\[
\mathbb E_{Y\sim q^{N-1}}[B(Y)^2]
\le
\frac{2N}{N+1}.
\]

\noindent Once we identify a nonnegative function \(B\) satisfying the above
conditions, the marginal divergence satisfies
\[
\begin{aligned}
D_{\mathrm{KL}}(\widetilde r_Y\|q^{N-1})
&=
\int_{\mathcal X^{N-1}}
\widetilde r_Y(y)
\log\frac{\widetilde r_Y(y)}{q^{N-1}(y)}
\,dy
\\
&=
\int_{\mathcal X^{N-1}}
\log B(y)\,
\widetilde r_Y(y)\,dy
\\
&=
\mathbb{E}_{Y\sim \widetilde r_Y}[\log B(Y)].
\end{aligned}
\]
Since \(\widetilde r_Y\) is a probability density and \(\log\) is
concave, Jensen's inequality gives
\[
\begin{aligned}
D_{\mathrm{KL}}(\widetilde r_Y\|q^{N-1})
&\le
\log \mathbb{E}_{Y\sim \widetilde r_Y}[B(Y)]
\\
&=
\log\left(
\int_{\mathcal X^{N-1}}
B(y)\widetilde r_Y(y)\,dy
\right)
\\
&=
\log\left(
\int_{\mathcal X^{N-1}}
B(y)^2q^{N-1}(y)\,dy
\right)
\\
&=
\log
\mathbb E_{Y\sim q^{N-1}}[B(Y)^2].
\end{aligned}
\]
Finally, using the second-moment bound, we obtain
\[
D_{\mathrm{KL}}(\widetilde r_Y\|q^{N-1})
\le
\log\left(\frac{2N}{N+1}\right)
\le
\log 2.
\]

\noindent It therefore remains to identify \(B(Y)\) and establish the
second-moment bound.

\subsection*{Bounding the Second Moment}

We first carefully characterize the function \(B(y)\). Let
\[
Y=(X_2,\ldots,X_N),
\qquad
t(Y)
\defeq 
\sum_{i=2}^N w(X_i).
\]
The law that \(Y\) follows will be specified accordingly later. For a fixed
realization
\[
Y=x_{2:N},
\]
we write
\[
t(x_{2:N})
=
\sum_{i=2}^N w(x_i).
\]

\noindent To obtain the marginal density of \(Y\) under \(\widetilde r\),
we integrate out the first coordinate:
\[
\widetilde r_Y(x_{2:N})
=
\int_{\mathcal X}
\widetilde r(x,x_{2:N})\,dx.
\]
Since
\[
\frac{\widetilde r(x,x_{2:N})}
{r(x,x_{2:N})}
=
\frac{Nw(x)}
{w(x)+t(x_{2:N})},
\]
we have
\[
\widetilde r(x,x_{2:N})
=
\frac{Nw(x)}
{w(x)+t(x_{2:N})}
r(x,x_{2:N}).
\]
Therefore,
\[
\widetilde r_Y(x_{2:N})
=
\int_{\mathcal X}
\frac{Nw(x)}
{w(x)+t(x_{2:N})}
r(x,x_{2:N})\,dx.
\]

\noindent Because \(r=q^N\),
\[
r(x,x_{2:N})
=
q(x)\,q^{N-1}(x_{2:N})
=
q(x)\prod_{i=2}^N q(x_i).
\]
Substituting this factorization gives
\[
\begin{aligned}
\widetilde r_Y(x_{2:N})
&=
\int_{\mathcal X}
\frac{Nw(x)}
{w(x)+t(x_{2:N})}
q(x)\,q^{N-1}(x_{2:N})\,dx
\\
&=
\left[
N\int_{\mathcal X}
\frac{w(x)}
{w(x)+t(x_{2:N})}
q(x)\,dx
\right]
q^{N-1}(x_{2:N}),
\end{aligned}
\]
where \(q^{N-1}(x_{2:N})\) may be taken outside the integral because it
does not depend on \(x\).

\noindent Dividing by \(q^{N-1}(x_{2:N})\), we obtain
\[
\frac{\widetilde r_Y(x_{2:N})}
{q^{N-1}(x_{2:N})}
=
N\int_{\mathcal X}
\frac{w(x)}
{w(x)+t(x_{2:N})}
q(x)\,dx.
\]

\noindent We therefore define, for \(y\in\mathcal X^{N-1}\),
\begin{align}
B(y)
&\defeq 
N\int_{\mathcal X}
\frac{w(x)}
{w(x)+t(y)}
q(x)\,dx
\notag\\
&=
N\mathbb E_{X\sim q}
\left[
\frac{w(X)}
{w(X)+t(y)}
\right].
\tag{19}
\end{align}

\paragraph{Second Moment Bound.}

Recall that we want to prove
\[
\mathbb E_{Y\sim q^{N-1}}[B(Y)^2]
\le
\frac{2N}{N+1}.
\]

\noindent For a fixed \(y\in\mathcal X^{N-1}\), let
\[
U,V\overset{\mathrm{i.i.d.}}{\sim}q.
\]
By the definition of \(B(y)\),
\[
\begin{aligned}
B(y)^2
&=
N^2
\mathbb E_{U\sim q}
\left[
\frac{w(U)}{w(U)+t(y)}
\right]
\mathbb E_{V\sim q}
\left[
\frac{w(V)}{w(V)+t(y)}
\right]
\\
&=
N^2
\mathbb E_{U,V\overset{\mathrm{i.i.d.}}{\sim}q}
\left[
\frac{w(U)w(V)}
{\bigl(w(U)+t(y)\bigr)\bigl(w(V)+t(y)\bigr)}
\right],
\end{aligned}
\tag{20}
\]
where the second equality follows from the independence of \(U\) and
\(V\).

\noindent We now take expectation with respect to
\[
Y\sim q^{N-1},
\]
independently of \(U\) and \(V\). Thus, under the product law
\[
q^{N-1}\otimes q\otimes q,
\]
the random variables \(Y\), \(U\), and \(V\) are mutually independent.
Consequently,
\[
\begin{aligned}
\mathbb E_{Y\sim q^{N-1}}[B(Y)^2]
&=
N^2
\mathbb E_{\substack{
Y\sim q^{N-1}\\
U,V\overset{\mathrm{i.i.d.}}{\sim}q
}}
\left[
\frac{w(U)w(V)}
{\bigl(w(U)+t(Y)\bigr)\bigl(w(V)+t(Y)\bigr)}
\right].
\end{aligned}
\tag{21}
\]

\noindent Now, we note that, for any positive $\alpha,\beta,\gamma$, we have
\[
\frac{\alpha\beta}{(\alpha+\gamma)(\beta+\gamma)}
\le
\frac{\alpha\beta}{(\alpha+\beta+\gamma)(\alpha+\gamma)}
+
\frac{\alpha\beta}{(\alpha+\beta+\gamma)(\beta+\gamma)}.
\]

\begin{center}
\fbox{
\begin{minipage}{0.92\linewidth}
\textbf{Proof.}
For any \(\alpha,\beta,\gamma>0\),
\[
\begin{aligned}
&\frac{\alpha\beta}
{(\alpha+\beta+\gamma)(\alpha+\gamma)}
+
\frac{\alpha\beta}
{(\alpha+\beta+\gamma)(\beta+\gamma)}
\\
&\qquad=
\frac{\alpha\beta}{\alpha+\beta+\gamma}
\left(
\frac{1}{\alpha+\gamma}
+
\frac{1}{\beta+\gamma}
\right)
\\
&\qquad=
\frac{\alpha\beta(\alpha+\beta+2\gamma)}
{(\alpha+\beta+\gamma)(\alpha+\gamma)(\beta+\gamma)}.
\end{aligned}
\]
Therefore, the desired inequality is equivalent to
\[
1
\le
\frac{\alpha+\beta+2\gamma}
{\alpha+\beta+\gamma},
\]
which holds because \(\gamma>0\). Hence,
\[
\frac{\alpha\beta}
{(\alpha+\gamma)(\beta+\gamma)}
\le
\frac{\alpha\beta}
{(\alpha+\beta+\gamma)(\alpha+\gamma)}
+
\frac{\alpha\beta}
{(\alpha+\beta+\gamma)(\beta+\gamma)}.
\]
\end{minipage}
}
\end{center}

\noindent Applying the above inequality, we obtain
\begin{align*}
\mathbb E_{Y\sim q^{N-1}}[B(Y)^2]
&\le
N^2
\mathbb E_{\substack{
Y\sim q^{N-1}\\
U,V\overset{\mathrm{i.i.d.}}{\sim}q
}}
\Bigg[
\frac{w(U)w(V)}
{(w(U)+w(V)+t(Y))(w(U)+t(Y))}
\\
&\hspace{5.5cm}
+
\frac{w(U)w(V)}
{(w(U)+w(V)+t(Y))(w(V)+t(Y))}
\Bigg].
\end{align*}

\noindent For the first expectation term, we use
\[
\frac{w(U)w(V)}
{\bigl(w(U)+w(V)+t(Y)\bigr)\bigl(w(U)+t(Y)\bigr)}
=
\frac{w(U)}{w(U)+t(Y)}
-
\frac{w(U)}{w(U)+w(V)+t(Y)}.
\]
Hence,
\[
\begin{aligned}
&\mathbb E_{\substack{
Y\sim q^{N-1}\\
U,V\overset{\mathrm{i.i.d.}}{\sim}q
}}
\left[
\frac{w(U)w(V)}
{\bigl(w(U)+w(V)+t(Y)\bigr)\bigl(w(U)+t(Y)\bigr)}
\right]
\\
&\qquad=
\mathbb E_{\substack{
U,X_2,\ldots,X_N\overset{\mathrm{i.i.d.}}{\sim}q
}}
\left[
\frac{w(U)}
{w(U)+\sum_{i=2}^N w(X_i)}
\right]
\\
&\qquad\quad-
\mathbb E_{\substack{
U,V,X_2,\ldots,X_N\overset{\mathrm{i.i.d.}}{\sim}q
}}
\left[
\frac{w(U)}
{w(U)+w(V)+\sum_{i=2}^N w(X_i)}
\right].
\end{aligned}
\]

\noindent By exchangeability of the \(N\) i.i.d. random variables
\[
U,X_2,\ldots,X_N,
\]
we have
\[
\mathbb E_{\substack{
U,X_2,\ldots,X_N\overset{\mathrm{i.i.d.}}{\sim}q
}}
\left[
\frac{w(U)}
{w(U)+\sum_{i=2}^N w(X_i)}
\right]
=
\frac{1}{N}.
\]
Indeed, all \(N\) normalized weights have the same expectation and their
sum is equal to \(1\). Similarly, by exchangeability of the \(N+1\) i.i.d.
random variables
\[
U,V,X_2,\ldots,X_N,
\]
we have
\[
\mathbb E_{\substack{
U,V,X_2,\ldots,X_N\overset{\mathrm{i.i.d.}}{\sim}q
}}
\left[
\frac{w(U)}
{w(U)+w(V)+\sum_{i=2}^N w(X_i)}
\right]
=
\frac{1}{N+1}.
\]

\noindent Therefore,
\[
\begin{aligned}
&\mathbb E_{\substack{
Y\sim q^{N-1}\\
U,V\overset{\mathrm{i.i.d.}}{\sim}q
}}
\left[
\frac{w(U)w(V)}
{\bigl(w(U)+w(V)+t(Y)\bigr)\bigl(w(U)+t(Y)\bigr)}
\right]
\\
&\qquad=
\frac{1}{N}-\frac{1}{N+1}
=
\frac{1}{N(N+1)}.
\end{aligned}
\]

\noindent By symmetry between \(U\) and \(V\), the second expectation term is also
equal to
\[
\frac{1}{N(N+1)}.
\]
Consequently,
\[
\mathbb E_{Y\sim q^{N-1}}[B(Y)^2]
\le
N^2\left(
\frac{1}{N(N+1)}
+
\frac{1}{N(N+1)}
\right)
=
\frac{2N}{N+1}
\le 2,
\]
which concludes the result.

\section{Proof of Theorem \texorpdfstring{\ref{prop:2nd_result}}{prop:2nd\_result}}\label{thm:2}

\subsection{\texorpdfstring{
Proof of \texorpdfstring{$D_{\mathrm{KL}}(p_N\|q) \le \mathbb{E}[\widehat D_N]$}{DKL(pN||q) <= E[DN]} 
}{
Proof of D_KL(p_N||q) <= E[D_N]
}}\label{app:a}
Recall that the marginal distribution $p_N$ of the resampled particle
$X_I$ has the explicit form
\[
p_N(x)
=
Nq(x)\,
\mathbb{E}_{X_{2:N}\sim q^{\otimes(N-1)}}
\left[
\frac{w(x)}
{w(x)+\sum_{j=2}^N w(X_j)}
\right].
\]
Consider the joint distribution of the candidate pool and selected index,
\[
p^\star(x_{1:N},i)
=
\left(
\prod_{j=1}^N q(x_j)
\right)
\lambda_i(x_{1:N}),
\]
and the reference joint distribution
\[
q^\star(x_{1:N},i)
=
\left(
\prod_{j=1}^N q(x_j)
\right)
\frac{1}{N}.
\]
Under both distributions, the candidate pool
$X_{1:N}$ is drawn i.i.d.\ from $q$; the only difference is that under
$p^\star$ the index $I$ is sampled according to the normalized importance
weights, whereas under $q^\star$ it is sampled uniformly. Their KL divergence is
\[
\begin{aligned}
D_{\mathrm{KL}}(p^\star\|q^\star)
&=
\sum_{i=1}^N
\int
p^\star(x_{1:N},i)
\log
\frac{p^\star(x_{1:N},i)}
{q^\star(x_{1:N},i)}
\,dx_{1:N}
\\
&=
\sum_{i=1}^N
\int
\left(
\prod_{j=1}^N q(x_j)
\right)
\lambda_i(x_{1:N})
\log\bigl(N\lambda_i(x_{1:N})\bigr)
\,dx_{1:N}
\\
&=
\mathbb{E}_{X_{1:N}\sim q^{\otimes N}}
\left[
\sum_{i=1}^N
\lambda_i\log(N\lambda_i)
\right]
\\
&=
\mathbb{E}[\widehat D_N].
\end{aligned}
\]

\noindent Now apply the mapping
\[
(X_{1:N},I)\mapsto X_I.
\]
Under $p^\star$, the output $X_I$ has marginal distribution $p_N$.
Under $q^\star$, since $I$ is uniform and independent of the i.i.d.\
candidate pool, we have $X_I\sim q$. Therefore, by the data-processing inequality \citep{polyanskiy2025information},
\[
D_{\mathrm{KL}}(p_N\|q)
\le
D_{\mathrm{KL}}(p^\star\|q^\star)
=
\mathbb{E}[\widehat D_N],
\]
which completes the proof.

\subsection{\texorpdfstring{
Proof of $D_{\mathrm{KL}}(p_N\|q) \le D_{\mathrm{KL}}(p\|q)$
}{
Proof of D_KL(p_N||q) <= D_KL(p||q)
}}
Let $X\sim q$ and write
\[
W\defeq w(X),
\qquad
\mathbb{E}_q[W]=1.
\]
For $N=1$, we have $p_1=q$, and hence
\[
D_{\mathrm{KL}}(p_1\|q)=0
\le
D_{\mathrm{KL}}(p\|q).
\]
We therefore assume $N\ge2$ for the remainder of the proof. By exchangeability of $X_1,\ldots,X_N$, the density (or probability mass
function) of the resampled particle is
\[
\begin{aligned}
p_N(x)
&=
Nq(x)\,
\mathbb{E}_{X_{2:N}\sim q^{\otimes(N-1)}}
\left[
\frac{w(x)}
{w(x)+\sum_{j=2}^N w(X_j)}
\right]
\\
&=
Nq(x)\,
\mathbb{E}_{X_{2:N}\sim q^{\otimes(N-1)}}
\left[
\frac{\frac{p(x)}{q(x)}}
{\frac{p(x)}{q(x)}
+\sum_{j=2}^N\frac{p(X_j)}{q(X_j)}}
\right].
\end{aligned}
\]
Equivalently,
\[
\frac{p_N(x)}{q(x)}
=
h_N(w(x)),
\]
where
\[
h_N(t)
\defeq 
N\,
\mathbb{E}
\left[
\frac{t}
{t+\sum_{j=2}^N W_j}
\right],
\qquad
W_j\defeq w(X_j).
\]
Since $p_N$ and $p$ are probability distributions,
\[
\mathbb{E}_q[h_N(W)]
=
\mathbb{E}_q[W]
=
1.
\]

\noindent We will show that there exists $c>0$ such that
\begin{equation}
\min\{t,c\}
\le
h_N(t)
\le
\max\{t,c\},
\qquad t>0.
\label{sandwich_hN}
\end{equation}
Applying Lemma~\ref{lemma_mean} with
\[
T=W,
\qquad
g(t)=h_N(t),
\qquad
f(t)=t\log t,
\]
then gives
\[
\mathbb{E}_q
\left[
h_N(W)\log h_N(W)
\right]
\le
\mathbb{E}_q
\left[
W\log W
\right].
\]
Since
\[
\frac{p_N(X)}{q(X)}=h_N(W),
\]
the two sides are precisely
\[
D_{\mathrm{KL}}(p_N\|q)
\qquad\text{and}\qquad
D_{\mathrm{KL}}(p\|q),
\]
respectively. It therefore remains only to prove~\eqref{sandwich_hN}. First, observe that $h_N$ is strictly increasing. Indeed,
differentiating under the expectation gives
\[
h_N'(t)
=
N\,
\mathbb{E}
\left[
\frac{
\sum_{j=2}^N W_j
}{
\left(
t+\sum_{j=2}^N W_j
\right)^2
}
\right]
>0,
\]
since $W_j>0$ almost surely. We next establish the following two claims:
\begin{enumerate}
    \item There exists $c>0$ such that
    \[
    h_N(c)=c.
    \]
    \item We have
    \[
    h_N(t)\ge t
    \qquad\text{for }t\le c,
    \]
    and
    \[
    h_N(t)\le t
    \qquad\text{for }t\ge c.
    \]
\end{enumerate}
Once these claims are established, the strict monotonicity of $h_N$ and
$h_N(c)=c$ imply
\[
t\le h_N(t)\le c,
\qquad t\le c,
\]
and
\[
c\le h_N(t)\le t,
\qquad t\ge c.
\]
This is exactly~\eqref{sandwich_hN}.

\paragraph{Proving the claims.}
We first note that $h_N$ is continuous on $(0,\infty)$. Indeed, let
$t_k\to t>0$. Since
\[
\frac{t_k}{t_k+\sum_{j=2}^N W_j}
\longrightarrow
\frac{t}{t+\sum_{j=2}^N W_j}
\]
and the integrand is bounded between $0$ and $1$, the dominated
convergence theorem gives
\[
h_N(t_k)\longrightarrow h_N(t).
\]
Consequently, $h_N(t)/t$ is also continuous on $(0,\infty)$.

Now observe that
\[
\frac{h_N(t)}{t}
=
N\,
\mathbb{E}
\left[
\frac{1}
{t+\sum_{j=2}^N W_j}
\right].
\]
For every realization of $W_2,\ldots,W_N$, the function
\[
t\mapsto
\frac{1}
{t+\sum_{j=2}^N W_j}
\]
is strictly decreasing. Taking expectations therefore shows that
$h_N(t)/t$ is strictly decreasing on $(0,\infty)$.

Moreover,
\[
h_N(t)-t
=
t\left(
\frac{h_N(t)}{t}-1
\right).
\]
Since $t>0$, the sign of $h_N(t)-t$ is determined by whether
$h_N(t)/t$ is above or below $1$. Thus, $h_N(t)-t$ can cross zero at
most once, and any crossing must be from positive to negative. On the other hand,
\[
\mathbb{E}_q[h_N(W)-W]
=
\mathbb{E}_q[h_N(W)]
-
\mathbb{E}_q[W]
=
0.
\]
If $h_N(W)=W$ almost surely, then any value in the support of $W$ is a
fixed point and the claims follow immediately. Otherwise,
$h_N(W)-W$ cannot be strictly positive almost surely or strictly negative
almost surely, since its expectation is zero. Hence there exist values
$t_1,t_2$ in the support of $W$ such that
\[
h_N(t_1)>t_1,
\qquad
h_N(t_2)<t_2.
\]
Since $h_N(t)/t$ is strictly decreasing, we must have $t_1<t_2$.
By continuity, there exists $c\in(t_1,t_2)$ such that
\[
h_N(c)=c.
\]

\noindent Finally, since $h_N(t)/t$ is decreasing and
\[
\frac{h_N(c)}{c}=1,
\]
we have
\[
\frac{h_N(t)}{t}\ge1
\qquad\text{for }t\le c,
\]
and
\[
\frac{h_N(t)}{t}\le1
\qquad\text{for }t\ge c.
\]
Equivalently,
\[
h_N(t)\ge t
\qquad\text{for }t\le c,
\]
and
\[
h_N(t)\le t
\qquad\text{for }t\ge c.
\]
This proves the two claims and completes the proof.

\section{\texorpdfstring{
Proof of  Theorem \ref{prop:snis-lower}
}{
Proof of Theorem~\ref{prop:snis-lower}
}}\label{thm3:proof}

Recall that, for \(y\) in the support of \(q\),
\[
\frac{p_N(y)}{q(y)}
=
N\,
\mathbb{E}_{X_{2:N}\sim q^{\otimes(N-1)}}
\left[
\frac{w(y)}
{w(y)+\sum_{j=2}^N w(X_j)}
\right].
\]
For \(M\ge1\), define
\[
\widehat \theta_M(y)
=
\frac{N}{M}
\sum_{m=1}^M
\frac{w(y)}
{w(y)+\sum_{j=2}^N w(X_j^{(m)})},
\qquad
X_j^{(m)}\overset{\mathrm{i.i.d.}}{\sim}q,
\]
where the auxiliary samples \(\{X_j^{(m)}\}_{j,m}\) are independent across
\(j,m\) and independent of \(Y\sim p_N\). For every fixed \(y\),
\[
\begin{aligned}
\mathbb{E}_{\mathrm{aux}}
\left[
\widehat \theta_M(y)
\right]
&=
\frac{N}{M}
\sum_{m=1}^M
\mathbb{E}_{\mathrm{aux}}
\left[
\frac{w(y)}
{w(y)+\sum_{j=2}^N w(X_j^{(m)})}
\right]
\\
&=
N\,
\mathbb{E}_{X_{2:N}\sim q^{\otimes(N-1)}}
\left[
\frac{w(y)}
{w(y)+\sum_{j=2}^N w(X_j)}
\right]
\\
&=
\frac{p_N(y)}{q(y)}.
\end{aligned}
\]
Thus, \(\widehat \theta_M(y)\) is an unbiased estimator of the density ratio
\(p_N(y)/q(y)\). Since the logarithm is concave, Jensen's inequality gives,
for every fixed \(y\),
\[
\mathbb{E}_{\mathrm{aux}}
\left[
\log \widehat \theta_M(y)
\right]
\le
\log
\mathbb{E}_{\mathrm{aux}}
\left[
\widehat \theta_M(y)
\right]
=
\log\frac{p_N(y)}{q(y)}.
\]
Taking expectation with respect to \(Y\sim p_N\) yields
\[
\begin{aligned}
\mathbb{E}_{Y\sim p_N}
\left[
\mathbb{E}_{\mathrm{aux}}
\left[
\log \widehat \theta_M(Y)
\right]
\right]
&\le
\mathbb{E}_{Y\sim p_N}
\left[
\log\frac{p_N(Y)}{q(Y)}
\right]
\\
&=
D_{\mathrm{KL}}(p_N\|q).
\end{aligned}
\]
Therefore,
\[
\mathbb{E}_{Y\sim p_N}
\left[
\mathbb{E}_{\mathrm{aux}}
\left[
\log \widehat \theta_M(Y)
\right]
\right]
\le
D_{\mathrm{KL}}(p_N\|q).
\]

\paragraph{Convergence.}
For every fixed $y$, the strong law of large numbers gives
\[
\widehat \theta_M(y)
\longrightarrow
\frac{p_N(y)}{q(y)}
\qquad\text{almost surely as }M\to\infty,
\]
and hence, by continuity of the logarithm,
\[
\log \widehat \theta_M(y)
\longrightarrow
\log\frac{p_N(y)}{q(y)}
\qquad\text{almost surely}.
\]
Therefore, for every fixed $B$,
\[
\widehat L_{B,M}
\longrightarrow
\max\left\{
0,
\frac{1}{B}\sum_{b=1}^B
\log\frac{p_N(Y_b)}{q(Y_b)}
\right\}
\qquad\text{almost surely as }M\to\infty.
\]

\noindent Moreover, since $p_N(y)/q(y)\le N$, the positive part satisfies
\[
\mathbb{E}_{p_N}\!\left[
\left(\log\frac{p_N(Y)}{q(Y)}\right)_+
\right]
\le \log N.
\]
For the negative part, writing
\[
r(y)\defeq \frac{p_N(y)}{q(y)}
\]
and changing measure from $p_N$ to $q$ gives
\[
\begin{aligned}
\mathbb{E}_{p_N}\!\left[(-\log r(Y))_+\right]
&=
\mathbb{E}_q\!\left[
r(Y)(-\log r(Y))\mathbf{1}\{r(Y)<1\}
\right]
\\
&\le \frac{1}{e},
\end{aligned}
\]
where we used $-t\log t\le 1/e$ for $t\in(0,1)$. Hence,
\[
\mathbb{E}_{p_N}\!\left[
\left|\log\frac{p_N(Y)}{q(Y)}\right|
\right]
\le \log N+\frac{1}{e}<\infty.
\]
The strong law of large numbers therefore yields
\[
\frac{1}{B}\sum_{b=1}^B
\log\frac{p_N(Y_b)}{q(Y_b)}
\longrightarrow
D_{\mathrm{KL}}(p_N\|q)
\qquad\text{almost surely as }B\to\infty.
\]
Since $D_{\mathrm{KL}}(p_N\|q)\ge 0$ and
$x\mapsto\max\{0,x\}$ is continuous,
\[
\max\left\{
0,
\frac{1}{B}\sum_{b=1}^B
\log\frac{p_N(Y_b)}{q(Y_b)}
\right\}
\longrightarrow
D_{\mathrm{KL}}(p_N\|q)
\qquad\text{almost surely}.
\]
Therefore,
\[
\lim_{B\to\infty}\lim_{M\to\infty}\widehat L_{B,M}
=
D_{\mathrm{KL}}(p_N\|q)
\qquad\text{almost surely}.
\]
\section{Proof for Soft Best-of-N Optimality}\label{sBoN_tilted_policy}
For completeness, we verify that the solution to the Soft BoN optimization
problem
\[
\max_{\pi}
\;
\mathbb{E}_{Y\sim\pi(\cdot\mid x)}
\left[
r(x,Y)
\right]
\qquad
\text{subject to}
\qquad
D_{\mathrm{KL}}
\left(
\pi(\cdot\mid x)
\middle\|
\pi_0(\cdot\mid x)
\right)
\le \varepsilon
\]
is given by an exponential tilt of the base policy
\citep{verdun2025soft}. In particular, for an appropriate inverse
temperature $\beta>0$, the optimal policy takes the form
\[
\pi_\beta(y\mid x)
=
\frac{
\pi_0(y\mid x)e^{\beta r(x,y)}
}{
Z_\beta(x)
},
\qquad
Z_\beta(x)
=
\mathbb{E}_{Y\sim\pi_0(\cdot\mid x)}
\left[
e^{\beta r(x,Y)}
\right].
\]
\begin{proof}
    The objective is linear, and hence concave, in the policy $\pi$.
To see this directly, let $\pi_1$ and $\pi_2$ be two policies and let
$\theta\in[0,1]$. Define the mixture policy
\[
\pi_{\mathrm{mixture}}(y;\theta)
=
\theta\pi_1(y)+(1-\theta)\pi_2(y).
\]
Then
\[
\begin{aligned}
\mathbb{E}_{Y\sim\pi_{\mathrm{mixture}}(y;\theta)}[r(Y)]
&=
\int r(y)\,\pi_{\mathrm{mixture}}(y;\theta)\,dy\\
&=
\int r(y)
\left[
\theta\pi_1(y)+(1-\theta)\pi_2(y)
\right]dy\\
&=
\theta
\int r(y)\pi_1(y)\,dy
+
(1-\theta)
\int r(y)\pi_2(y)\,dy\\
&=
\theta
\mathbb{E}_{Y\sim\pi_1}[r(Y)]
+
(1-\theta)
\mathbb{E}_{Y\sim\pi_2}[r(Y)].
\end{aligned}
\]
Thus the concavity inequality holds with equality, so the expected reward
is linear (and therefore both concave and convex) in $\pi$. Moreover, $D_{\mathrm{KL}}(\pi\|\pi_0)$ is convex in $\pi$ for fixed
$\pi_0$, so the constraint
\[
D_{\mathrm{KL}}(\pi\|\pi_0)\le\varepsilon
\]
defines a convex feasible set. Hence maximizing the expected reward over
this set is a convex optimization problem. Suppose the KL constraint is active, with multiplier $\eta>0$. The Lagrangian is
\[
\mathcal L(\pi,\eta,\alpha)
=
\int r(y)\pi(y)\,dy
-\eta
\left(
\int \pi(y)\log\frac{\pi(y)}{\pi_0(y)}\,dy-\varepsilon
\right)
+\alpha\left(\int\pi(y)\,dy-1\right).
\]

\noindent Stationarity with respect to $\pi(y)$ yields
\[
r(y)
-\eta\left(
\log\frac{\pi(y)}{\pi_0(y)}+1
\right)
+\alpha=0,
\]
and therefore
\[
\pi(y)\propto \pi_0(y)e^{r(y)/\eta}.
\]
Writing $\beta=1/\eta$ gives
\[
\pi_\beta(y)
=
\frac{\pi_0(y)e^{\beta r(y)}}
{\mathbb{E}_{Y\sim\pi_0}[e^{\beta r(Y)}]}.
\]
The value of $\beta$ is chosen such that
$D_{\mathrm{KL}}(\pi_\beta\|\pi_0)=\varepsilon$.
\end{proof}

\section{Consistency of Soft Best-of-\texorpdfstring{$N$}{N}}
\label{sBoN consistency}

Fix an input $x$ and assume
\[
0<Z_\beta(x)
=
\mathbb{E}_{Y\sim\pi_0(\cdot\mid x)}
\left[
e^{\beta r(x,Y)}
\right]
<\infty.
\]
Let
\[
Y_1,\ldots,Y_N
\overset{\mathrm{i.i.d.}}{\sim}
\pi_0(\cdot\mid x),
\]
and recall that, conditional on the candidate set $Y_{1:N}$, SBoN samples
\[
I\mid Y_{1:N}
\sim
\operatorname{Cat}(\lambda_1,\ldots,\lambda_N),
\]
where
\[
\begin{aligned}
\lambda_i
=
\frac{
\pi_\beta(Y_i\mid x)/\pi_0(Y_i\mid x)
}{
\sum_{j=1}^N
\pi_\beta(Y_j\mid x)/\pi_0(Y_j\mid x)
}=
\frac{
e^{\beta r(x,Y_i)}
}{
\sum_{j=1}^N e^{\beta r(x,Y_j)}
},
\qquad i\in[N].
\end{aligned}
\]

\noindent For any fixed output $y$, the conditional probability that the returned
response equals $y$ is
\[
\begin{aligned}
\mathbb{P}
\left(
Y_I=y
\mid
Y_{1:N}
\right)
&=
\sum_{i=1}^N
\mathbb{P}
\left(
I=i
\mid
Y_{1:N}
\right)
\mathbf{1}\{Y_i=y\}
\\
&=
\sum_{i=1}^N
\lambda_i
\mathbf{1}\{Y_i=y\}.
\end{aligned}
\]
Taking expectation over
\(
Y_{1:N}\sim\pi_0(\cdot\mid x)^{\otimes N}
\)
therefore gives
\[
\pi_{N,\beta}(y\mid x)
=
\mathbb{E}_{
Y_{1:N}\sim\pi_0(\cdot\mid x)^{\otimes N}
}
\left[
\sum_{i=1}^N
\lambda_i
\mathbf{1}\{Y_i=y\}
\right].
\]

\noindent Define the likelihood ratio
\[
w(y)
=
\frac{
\pi_\beta(y\mid x)
}{
\pi_0(y\mid x)
}
=
\frac{
e^{\beta r(x,y)}
}{
Z_\beta(x)
}.
\]
Since the common normalizing constant $Z_\beta(x)$ cancels in the
self-normalized weights,
\[
\lambda_i
=
\frac{
w(Y_i)
}{
\sum_{j=1}^N w(Y_j)
}.
\]
Hence,
\[
\pi_{N,\beta}(y\mid x)
=
\mathbb{E}_{
Y_{1:N}\sim\pi_0(\cdot\mid x)^{\otimes N}
}
\left[
\frac{
\sum_{i=1}^N
w(Y_i)\mathbf{1}\{Y_i=y\}
}{
\sum_{j=1}^N
w(Y_j)
}
\right].
\]
Equivalently,
\[
\pi_{N,\beta}(y\mid x)
=
\mathbb{E}_{
Y_{1:N}\sim\pi_0(\cdot\mid x)^{\otimes N}
}
\left[
\frac{
\frac{1}{N}\sum_{i=1}^N
w(Y_i)\mathbf{1}\{Y_i=y\}
}{
\frac{1}{N}\sum_{j=1}^N
w(Y_j)
}
\right].
\]

\noindent Because $w$ is the likelihood ratio,
\[
\mathbb{E}_{Y\sim\pi_0(\cdot\mid x)}
\left[
w(Y)
\right]
=
1.
\]
Therefore, by the strong law of large numbers,
\[
\frac{1}{N}
\sum_{j=1}^N
w(Y_j)
\longrightarrow
1
\qquad\text{a.s.}
\]
Moreover, for every fixed $y$,
\[
\begin{aligned}
\mathbb{E}_{Y\sim\pi_0(\cdot\mid x)}
\left[
w(Y)\mathbf{1}\{Y=y\}
\right]
&=
\pi_0(y\mid x)w(y)
\\
&=
\pi_\beta(y\mid x),
\end{aligned}
\]
so again by the strong law,
\[
\frac{1}{N}
\sum_{i=1}^N
w(Y_i)\mathbf{1}\{Y_i=y\}
\longrightarrow
\pi_\beta(y\mid x)
\qquad\text{a.s.}
\]
Hence,
\[
\frac{
\frac{1}{N}\sum_{i=1}^N
w(Y_i)\mathbf{1}\{Y_i=y\}
}{
\frac{1}{N}\sum_{j=1}^N
w(Y_j)
}
\longrightarrow
\pi_\beta(y\mid x)
\qquad\text{a.s.}
\]

\noindent The ratio inside the expectation is always between $0$ and $1$.
Therefore, by the dominated convergence theorem,
\[
\pi_{N,\beta}(y\mid x)
\longrightarrow
\pi_\beta(y\mid x)
\qquad
\text{for every }y.
\]
Thus, on the countable output space, we have
\(
\pi_{N,\beta}(\cdot\mid x)
\longrightarrow
\pi_\beta(\cdot\mid x)
\text{ as }N\to\infty.
\)
\section{Monotonicity of the Population KL in $\beta$}
\label{sec:beta-kl-monotonicity}

We show that the KL divergence of the Gibbs policy from the base policy is
nondecreasing in $\beta$. Fix $x$ and define
\[
Z_\beta(x)
\defeq 
\E_{Y\sim\pi_0(\cdot\mid x)}
\left[e^{\beta r(x,Y)}\right].
\]
Recall that
\[
\pi_\beta(y\mid x)
=
\frac{\pi_0(y\mid x)e^{\beta r(x,y)}}{Z_\beta(x)}.
\]
Assume that
\[
Z_\beta(x)
=
\E_{Y\sim\pi_0(\cdot\mid x)}
\left[e^{\beta r(x,Y)}\right]
<\infty
\]
for every $\beta\in\mathbb{R}$, then we have:
\[
\frac{\partial}{\partial\beta}\log Z_\beta(x)
=
\E_{Y\sim\pi_\beta(\cdot\mid x)}[r(x,Y)],
\]
and hence
\[
\frac{\partial^2}{\partial\beta^2}\log Z_\beta(x)
=
\mathrm{Var}_{Y\sim\pi_\beta(\cdot\mid x)}[r(x,Y)].
\]
Moreover,
\[
D_{\mathrm{KL}}\!\left(
\pi_\beta(\cdot\mid x)\|
\pi_0(\cdot\mid x)
\right)
=
\beta\E_{Y\sim\pi_\beta(\cdot\mid x)}[r(x,Y)]
-\log Z_\beta(x).
\]
Differentiating gives
\[
\frac{\partial}{\partial\beta}
D_{\mathrm{KL}}\!\left(
\pi_\beta(\cdot\mid x)\|
\pi_0(\cdot\mid x)
\right)
=
\beta\,
\mathrm{Var}_{Y\sim\pi_\beta(\cdot\mid x)}[r(x,Y)]
\geq 0.
\]
Thus, for $\beta\geq0$, the population KL divergence from the base policy is
nondecreasing in $\beta$.
\section{Proof of Corollary \ref{cor:sbon}}\label{proof_cor:sobn}
Set
\[
p=\pi_\beta(\cdot\mid x),
\qquad
q=\pi_0(\cdot\mid x).
\]
The corresponding likelihood ratio is
\[
w(y)
=
\frac{p(y)}{q(y)}
=
\frac{e^{\beta r(x,y)}}{Z_\beta(x)}.
\]
Therefore,
\[
\frac{w(Y_i)}
{\sum_{j=1}^N w(Y_j)}
=
\frac{e^{\beta r(x,Y_i)}}
{\sum_{j=1}^N e^{\beta r(x,Y_j)}}
=
\lambda_i.
\]
Hence SBoN is precisely the sampling-importance-resampling procedure of
our main results, with
\[
P_N=\pi_{N,\beta}(\cdot\mid x).
\]
The first two claims therefore follow immediately from Proposition \ref{prop:kl_estimator} and \ref{prop:2nd_result}. For the bounded-reward statement, for every $i$,
\[
\lambda_i
=
\frac{1}
{1+\sum_{j\ne i}
e^{\beta(r(x,Y_j)-r(x,Y_i))}}
\le
\frac{1}
{1+(N-1)e^{-\beta R_{\max}}},
\]
since
\[
r(x,Y_j)-r(x,Y_i)\ge -R_{\max}.
\]
It follows that
\[
\begin{aligned}
\widehat D_N
&=
\sum_{i=1}^N
\lambda_i\log(N\lambda_i)
\\
&\le
\sum_{i=1}^N
\lambda_i
\log
\frac{N}
{1+(N-1)e^{-\beta R_{\max}}}
\\
&=
\log
\frac{N}
{1+(N-1)e^{-\beta R_{\max}}},
\end{aligned}
\]
which proves the final claim.

\section{Experimental Validation}\label{sec:exp}

We first focus on the additive \(\log 2\) slack in \Cref{prop:kl_estimator}. 
Specifically, we investigate how large the finite-sample gap
\[
\Delta^*=\sup_N \{\mathbb{E}[\widehat D_N]-D_{\mathrm{KL}}(p\|q)\}
\]
can become across a range of distributions and sample sizes. Our experiments confirm that this gap can indeed be positive, while suggesting that the theoretical \(\log 2\approx 0.69\) constant is conservative: across all settings considered, we do not observe a gap larger than \(0.15\). This motivates the possibility that the universal additive constant in \Cref{prop:kl_estimator} can be potentially improved. We then turn to ~\Cref{prop:2nd_result} and empirically validate its bound across a range of distributions and sample sizes. All Monte Carlo estimates of $\mathbb{E}[\widehat D_N]$ are computed using $50{,}000$ independent replicates.

\subsection{Empirical Study of the Additive Slack}\label{sec:exp_slack}
\input{plots/bernouli}

\paragraph{Bernoulli distributions.}
Our first experiment considers a simple Bernoulli setting on $\{0,1\}$, with
\[
p(0)=1-\epsilon_1,
\qquad
p(1)=\epsilon_1,
\qquad
q(0)=\epsilon_2,
\qquad
q(1)=1-\epsilon_2,
\]
where $\epsilon_1,\epsilon_2\in(0,1/2)$. The corresponding importance weights are
\[
w(0)=\frac{1-\epsilon_1}{\epsilon_2},
\qquad
w(1)=\frac{\epsilon_1}{1-\epsilon_2}.
\]
In this setting, the expectation of the SNIS KL estimator can be evaluated
exactly. In particular, if
$M\sim\operatorname{Binomial}(N,\epsilon_2)$ denotes the number of zeros
among the $N$ samples drawn from $q$, then, conditional on $M=m$, the total
unnormalized importance weight is
\[
S_m
=
m w(0)+(N-m)w(1),
\]
and hence
\[
\mathbb{E}[\widehat D_N]
=
\sum_{m=0}^N
\binom{N}{m}
\epsilon_2^m(1-\epsilon_2)^{N-m}
\left[
\frac{m w(0)}{S_m}
\log\!\left(\frac{Nw(0)}{S_m}\right)
+
\frac{(N-m)w(1)}{S_m}
\log\!\left(\frac{Nw(1)}{S_m}\right)
\right].
\]
We evaluate the maximum finite-sample slack
\(
\Delta^*(\epsilon_1,\epsilon_2)
:=
\max_N
\left\{
\mathbb{E}[\widehat D_N]
-
D_{\mathrm{KL}}(p\|q)
\right\}
\)
over a logarithmically spaced grid of $(\epsilon_1,\epsilon_2)$, where the
maximum is taken over the considered sample sizes $N$. \Cref{fig:bernoulli-slack} shows that
positive slack occurs when $p$ is highly concentrated on state $0$, while
that state remains relatively unlikely under $q$. Across this Bernoulli family, the largest observed slack is approximately $0.0844$, substantially below the universal $\log 2$ bound. The maximum occurs at the boundary of the considered range of $\epsilon_1$; further decreasing $\epsilon_1$ increases the slack, but the observed values appear to converge to a finite limit.

\paragraph{Smooth long-tail Weibull.}

\input{plots/weibull}
\begin{figure}[t]
    \centering
    \includegraphics[width=\textwidth]{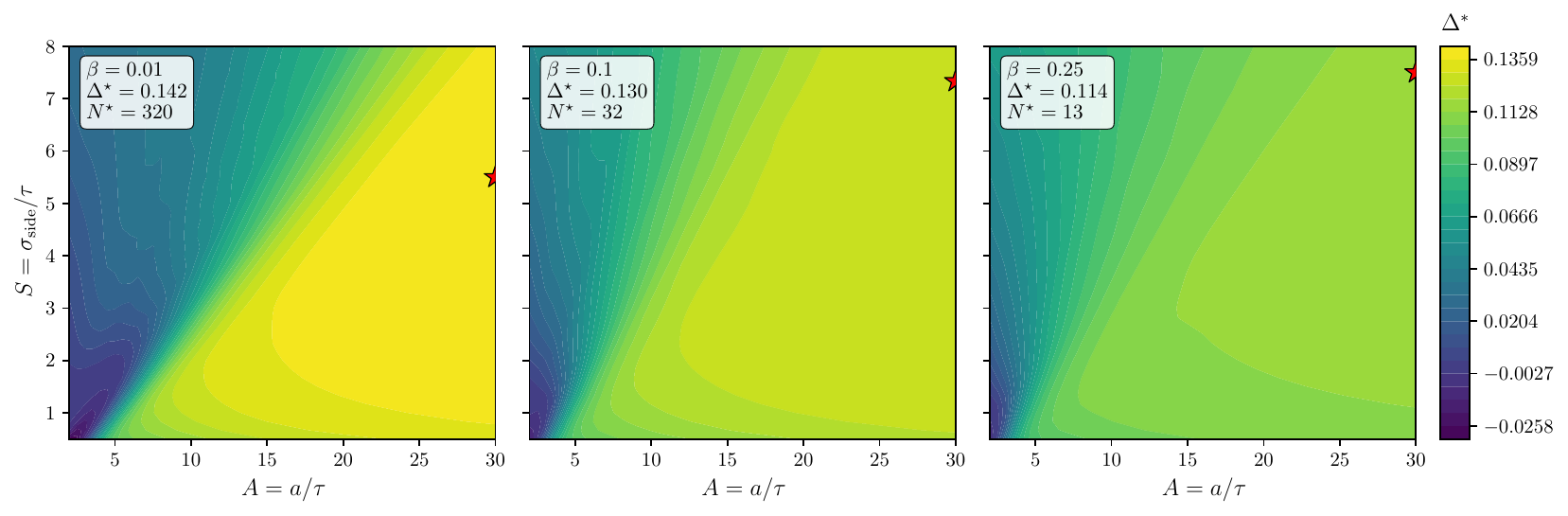}
    \caption{
        Profiled SNIS slack in the rare-bump setting.
        For each $(A,S)$ and fixed $\beta$, we plot
        \(
        \Delta^*(A,S;\beta)
        :=
        \max_{N}
        \left\{
        \mathbb{E}[\widehat D_N]
        -
        D_{\mathrm{KL}}(p\Vert q)
        \right\}.
        \)
        The red star marks the maximizing geometry for each value of $\beta$. Further decreasing $\beta$ below $0.01$ does not substantially increase the observed slack, suggesting that the maximum slack approaches a limiting value as $\beta\to 0$.
    }
    \label{fig:rare-bump-contours}
\end{figure}

We next consider whether differences in tail behavior between the target and proposal distributions contribute to positive finite-sample slack. To investigate this, we consider a two-parameter family of Weibull target densities.
\(
    p_{\mu,\nu} = \operatorname{Wei}(\mu,\nu),
\)
with
\[
    p_{\mu,\nu}(x)
    =
    \frac{\mu}{\nu}
    \left(\frac{x}{\nu}\right)^{\mu-1}
    \exp\left\{
        -\left(\frac{x}{\nu}\right)^\mu
    \right\},
    \qquad x\geq 0,
\]
and fix the proposal density to be
\(
    q=\operatorname{Wei}(1,1)=\operatorname{Exp}(1).
\)
We vary the shape parameter $\mu$ and scale parameter $\nu$ over a
two-dimensional grid.  In particular, $\mu<1$ corresponds to targets with
heavier tails than the exponential proposal.  The population KL divergence is available in closed form \citep{bauckhage2013computing}:
\[
\begin{aligned}
    D_{\mathrm{KL}}(p\Vert q)
    &=
    \log \mu-\log\nu
    +\gamma_{\mathrm E}
        \left(\frac{1}{\mu}-1\right) 
    +\nu
        \Gamma\!\left(1+\frac{1}{\mu}\right)
    -1,
\end{aligned}
\]
where $\Gamma(\cdot)$ denotes the Gamma function and $\gamma_{\mathrm E}$ is the Euler--Mascheroni constant.
This allows us to directly measure the finite-sample gap $\Delta^*(\mu,\nu)$. The results in \Cref{weibulla} show that the maximum slack remains close to zero across the considered parameter range, with sample sizes up to $N=2^{15}$. Moreover, the maximizing $N$ consistently occurs at the largest sample size considered, suggesting that the slack approaches zero from below as $N$ increases. This behavior is further illustrated in \Cref{weibullb}, where the slack remains negative across the considered sample sizes and increases toward zero. These results empirically suggest that heavier-tailed behavior of the target distribution alone does not lead to a  positive overshoot in the finite-sample slack.

\paragraph{Rare-bump setting.}
Let
\[
p=\mathcal N(0,\tau^2),\qquad
q=\beta\,\mathcal N(0,\tau^2)
+\frac{1-\beta}{2}\mathcal N(-a,\sigma^2)
+\frac{1-\beta}{2}\mathcal N(a,\sigma^2).
\]
The proposal places a small mass $\beta$ on a component matching the target, with the remaining mass split between two symmetric side modes. Since the problem is invariant to a common rescaling of $p$ and $q$, we set $\tau=1$ without loss of generality and parameterize the geometry by
\[
A=\frac{a}{\tau},\qquad S=\frac{\sigma}{\tau}.
\]
Thus the setting is completely specified by $(A,S,\beta)$, together with the number of importance samples $N$. 
As shown in \Cref{fig:rare-bump-contours}, the rare-bump setting produces substantially larger positive slack than the Bernoulli and Weibull examples. The largest observed slack increases as $\beta$ decreases, reaching approximately $0.142$ for $\beta=0.01$, attained at $N=320$. The maximizing geometry also lies near the boundary of the considered $(A,S)$ range: increasing $A$ and $S$ further continues to increase the slack, although the observed values appear to approach a limiting value. Similarly, further decreasing $\beta$ below $0.01$ does not substantially increase the observed slack, suggesting convergence to a limiting value as $\beta\to0$. Nevertheless, the observed slack remains well below the universal $\log 2$ bound.\\

\noindent Finally, a search over one million randomly generated discrete target--proposal pairs on a state space of size $256$ does not yield substantially larger slack. This further supports our theoretical guarantee, suggesting that the optimal universal constant may be considerably smaller than $\log 2$.

\subsection{
\texorpdfstring{
Validating the Lower and Upper Bounds on
$D_{\mathrm{KL}}(p_N\Vert q)$
}{
Validating the Lower and Upper Bounds on KL(pN || q)
}
}
\input{plots/second_bounds}

\Cref{fig:discrete-ordering} compares
$D_{\mathrm{KL}}(p_N\Vert q)$,
$\mathbb{E}[\widehat D_N]$,
$D_{\mathrm{KL}}(p\Vert q)$, and the lower bound $L_{N,4}$
for four discrete examples on a state space of size $L=8$.
For each sample size $N$, we compute
$D_{\mathrm{KL}}(p_N\Vert q)$ exactly, while
$\mathbb{E}[\widehat D_N]$ and $L_{N,4}$ are estimated via Monte Carlo.
The top row shows the corresponding probability masses of $p$ and $q$,
while the bottom row shows the induced divergence
$D_{\mathrm{KL}}(p_N\Vert q)$, the lower bound $L_{N,4}$, the expected
SNIS estimate $\mathbb{E}[\widehat D_N]$, and the population divergence
$D_{\mathrm{KL}}(p\Vert q)$.

The first three examples represent mild mismatch, moderate mismatch, and a
rare-mode setting, respectively.
In these cases, both finite-sample quantities increase toward the population
KL divergence, with $\mathbb{E}[\widehat D_N]$ consistently larger than
$D_{\mathrm{KL}}(p_N\Vert q)$.
The lower bound $L_{N,4}$ also closely tracks
$D_{\mathrm{KL}}(p_N\Vert q)$ across the examples, despite using only
$M=4$ auxiliary samples.
This behavior is consistent with the relatively small state space in this
experiment, for which the auxiliary estimator can accurately capture the
finite-sample density ratio with only a modest number of samples. The final example illustrates a qualitatively different behavior:
$\mathbb{E}[\widehat D_N]$ is non-monotone and exceeds the population
divergence for intermediate $N$, reaching approximately $0.603$ at $N=6$
compared with
$D_{\mathrm{KL}}(p\Vert q)\approx 0.559$.
Overall, the lower and upper bounds closely bracket the exact finite-sample
divergence in these experiments, providing further empirical validation of
our theoretical guarantees across a range of target--proposal configurations.

\begin{figure}[t]
\centering

\begin{tikzpicture}
\begin{groupplot}[
    group style={
        group size=3 by 1,
        horizontal sep=0.85cm,
    },
    width=0.335\linewidth,
    height=5.0cm,
    xmode=log,
    log basis x={2},
    xmin=3.5,
    xmax=70,
    xtick={4,8,16,32,64},
    xticklabels={4,8,16,32,64},
    xlabel={$N$},
    ymin=-4.3,
    ymax=3.2,
    tick align=outside,
    tick pos=left,
    axis line style={black!60, line width=0.35pt},
    tick style={black!60, line width=0.35pt},
    label style={font=\small},
    tick label style={font=\scriptsize},
    title style={font=\small},
    ymajorgrids=true,
    xmajorgrids=true,
    grid style={black!12},
]

\nextgroupplot[
    title={Total compute $=4N$},
    ylabel={Estimated value},
]

\addplot+[
    blue!75!black,
    dashed,
    mark=none,
    line width=1.0pt,
]
coordinates {
    (4,0.576634)
    (8,1.056609)
    (16,1.484965)
    (32,1.765131)
    (64,1.895003)
};

\addplot+[
    orange!90!black,
    mark=*,
    mark size=1.4pt,
    line width=0.9pt,
    error bars/.cd,
        y dir=both,
        y explicit,
]
coordinates {
    (4,1.004149) +- (0,0.349657) (0,0.305104)
    (8,1.438043) +- (0,0.370289) (0,0.353814)
    (16,1.752576) +- (0,0.346390) (0,0.376465)
    (32,1.910492) +- (0,0.298451) (0,0.340934)
    (64,1.965665) +- (0,0.222076) (0,0.257304)
};

\addplot+[
    green!55!black,
    mark=square*,
    mark size=1.35pt,
    line width=0.9pt,
    error bars/.cd,
        y dir=both,
        y explicit,
]
coordinates {
    (4,0.358099) +- (0,4.209789) (0,1.009669)
    (8,0.581573) +- (0,3.027761) (0,1.429993)
    (16,1.199084) +- (0,2.333092) (0,1.266707)
    (32,1.622562) +- (0,2.080430) (0,1.144244)
    (64,1.838957) +- (0,2.047497) (0,1.029778)
};

\nextgroupplot[
    title={Total compute $=8N$},
    yticklabels=\empty,
]

\addplot+[
    blue!75!black,
    dashed,
    mark=none,
    line width=1.0pt,
]
coordinates {
    (4,0.576634)
    (8,1.056609)
    (16,1.484965)
    (32,1.765131)
    (64,1.895003)
};

\addplot+[
    orange!90!black,
    mark=*,
    mark size=1.4pt,
    line width=0.9pt,
    error bars/.cd,
        y dir=both,
        y explicit,
]
coordinates {
    (4,1.005482) +- (0,0.228459) (0,0.208874)
    (8,1.448267) +- (0,0.251492) (0,0.245999)
    (16,1.748970) +- (0,0.257200) (0,0.286496)
    (32,1.908011) +- (0,0.215488) (0,0.237087)
    (64,1.972651) +- (0,0.153729) (0,0.169588)
};

\addplot+[
    green!55!black,
    mark=square*,
    mark size=1.35pt,
    line width=0.9pt,
    error bars/.cd,
        y dir=both,
        y explicit,
]
coordinates {
    (4,0.350052) +- (0,2.622591) (0,0.931024)
    (8,0.579005) +- (0,1.995892) (0,1.257058)
    (16,1.223187) +- (0,1.560429) (0,1.029460)
    (32,1.654941) +- (0,1.336729) (0,0.858525)
    (64,1.848119) +- (0,1.212688) (0,0.792878)
};

\nextgroupplot[
    title={Total compute $=16N$},
    yticklabels=\empty,
    legend style={
        at={(0.98,0.03)},
        anchor=south east,
        font=\scriptsize,
        draw=none,
        fill=white,
        fill opacity=0.9,
        text opacity=1,
    },
]

\addplot+[
    blue!75!black,
    dashed,
    mark=none,
    line width=1.0pt,
]
coordinates {
    (4,0.576634)
    (8,1.056609)
    (16,1.484965)
    (32,1.765131)
    (64,1.895003)
};
\addlegendentry{$D_{\mathrm{KL}}(p_N\Vert q)$}

\addplot+[
    orange!90!black,
    mark=*,
    mark size=1.4pt,
    line width=0.9pt,
    error bars/.cd,
        y dir=both,
        y explicit,
]
coordinates {
    (4,1.004055) +- (0,0.161686) (0,0.156432)
    (8,1.443183) +- (0,0.189546) (0,0.182389)
    (16,1.752048) +- (0,0.187619) (0,0.185824)
    (32,1.910949) +- (0,0.156566) (0,0.164860)
    (64,1.968844) +- (0,0.113253) (0,0.122506)
};
\addlegendentry{Upper}

\addplot+[
    green!55!black,
    mark=square*,
    mark size=1.35pt,
    line width=0.9pt,
    error bars/.cd,
        y dir=both,
        y explicit,
]
coordinates {
    (4,0.384266) +- (0,1.654697) (0,0.796664)
    (8,0.591802) +- (0,1.343806) (0,1.011502)
    (16,1.224173) +- (0,1.051464) (0,0.803204)
    (32,1.633730) +- (0,0.939119) (0,0.681295)
    (64,1.833570) +- (0,0.884365) (0,0.646394)
};
\addlegendentry{Best lower}

\end{groupplot}
\end{tikzpicture}

\caption{
Upper and lower estimators under equal computational budgets for the
bimodal-target, unimodal-proposal experiment.
The dashed curve shows the finite-sample divergence
$D_{\mathrm{KL}}(p_N\Vert q)$.
Points and error bars show the empirical means and central 95\% sampling
intervals of the upper and lower estimators across independent repetitions.
For each $(N,\text{budget})$, the auxiliary sample size $M$ for the lower
estimator is chosen to minimize the empirical interval width.
}
\label{fig:equal-compute}
\end{figure}
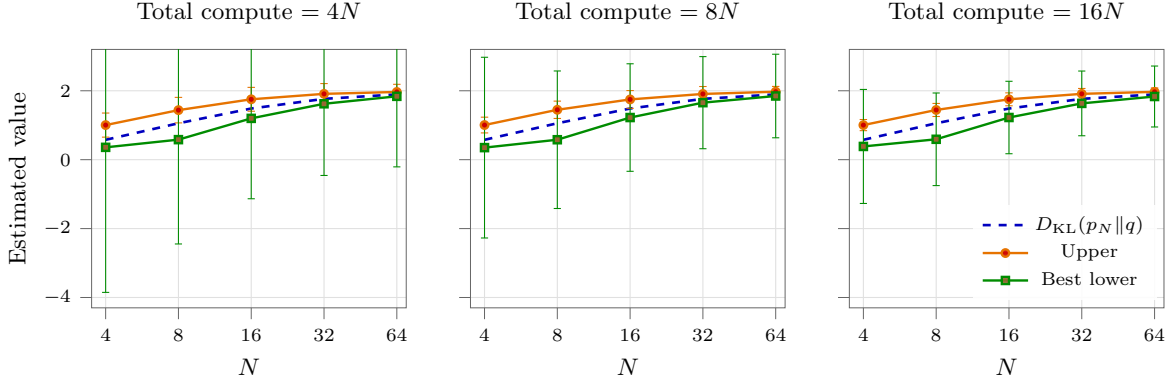

\paragraph{Finite-Compute Behavior.}We next compare the finite-compute behavior of the upper and lower estimators in a discrete experiment with $1000$ states. The target $p$ and proposal $q$ are constructed by evaluating Gaussian
mixture and Gaussian densities, respectively, on an evenly spaced grid
over $[-8,8]$ and normalizing the resulting probability masses. The target $p$ is a symmetric bimodal Gaussian mixture with component means $\pm 2.75$ and standard deviation $0.55$, while the proposal $q$ is a centered Gaussian with standard deviation $1.30$. For this pair,
\(
D_{\mathrm{KL}}(p\Vert q)\approx 1.994.
\)
We consider $N\in\{4,8,16,32,64\}$. For each $N$, we approximate $p_N$ using $500{,}000$ independent
auxiliary draws of $(X_2,\ldots,X_N)\sim q^{N-1}$ and use the resulting
density-ratio estimate to compute a high-accuracy reference value of
$D_{\mathrm{KL}}(p_N\Vert q)$. We likewise estimate the reference value of $\mathbb E[\widehat D_N]$ using a large number of independent SNIS samples. To compare the two estimators at equal computational cost, we consider total proposal-sampling budgets of $4N$, $8N$, and $16N$. For the upper estimator, a budget of $KN$ permits averaging $K$ independent realizations of $\widehat D_N$, each constructed from $N$ proposals. For the lower estimator, each outer sample $Y\sim p_N$ requires an SNIS pool of size $N$, together with $M$ independent auxiliary pools of size $N-1$. Thus, under the same total budget $KN$, the number of affordable outer samples is
\[
B
=
\left\lfloor
\frac{KN}{N+M(N-1)}
\right\rfloor.
\]
For each $(N,K)$ pair, we consider several values of $M$ and select the one yielding the smallest empirical central 95\% sampling-interval width. The intervals are estimated from $50,000$ independent repetitions of the complete estimator.

\Cref{fig:equal-compute} shows a clear difference in finite-compute behavior. The upper estimator is comparatively stable: its empirical sampling intervals shrink as the computational budget increases, and its mean remains close to $\mathbb E[\widehat D_N]$. In contrast, the lower estimator exhibits substantially larger variability under the same proposal-sampling budget. This is particularly pronounced for small and moderate $N$, where the logarithm of the auxiliary density-ratio estimate can be highly variable. Increasing $M$ improves the accuracy of the auxiliary approximation for a fixed outer sample, but it also reduces the number of outer samples $B$ that can be afforded under a fixed budget. In most of the settings considered here, the smallest empirical interval is therefore attained at $M=1$, indicating that allocating computation to additional outer samples is more effective than using many auxiliary samples per output. As $N$ increases, both $D_{\mathrm{KL}}(p_N\Vert q)$ and $\mathbb E[\widehat D_N]$ approach the population divergence. For example, at $N=64$ we obtain
\[
D_{\mathrm{KL}}(p_N\Vert q)\approx 1.895,
\qquad
\mathbb E[\widehat D_N]\approx 1.969,
\]
compared with $D_{\mathrm{KL}}(p\Vert q)\approx1.994$. The lower estimator also becomes more concentrated as $N$ and the total budget increase, but its sampling intervals remain much wider than those of the upper estimator. These results illustrate the practical distinction between the two sides of \Cref{prop:2nd_result}: the lower bound provides a way to estimate the implicit divergence $D_{\mathrm{KL}}(p_N\Vert q)$ from below, but requires an additional nested Monte Carlo procedure, whereas the upper bound is obtained directly from the normalized importance weights already produced by SNIS.
\paragraph{Shifted-Gaussian setup.}
We provide an additional finite-compute comparison using the shifted-Gaussian setup from \Cref{sbon_exp}, considering increasing levels of target--proposal mismatch, shown in \Cref{fig:finite-compute-gaussian-shift}. The results again show that increasing the computational budget substantially reduces the sampling variability of both estimators, while the lower estimator is considerably more variable than the upperbound.

\input{plots/additional_lb_compute}
\subsection{Additional Experiments on Soft BoN Compute}\label{add_sbon_exp}
We provide additional scenarios following the setup in the second part of \Cref{sbon_exp}. Specifically, we consider the discretized distribution pairs summarized in \Cref{tab:distribution-pairs}. The corresponding distributions and the performance of Soft BoN across different compute budgets are shown in \Cref{fig:snis-mode-tail,fig:snis-skew-irregular}. Overall, we find that $K=16$ already provides a sufficiently accurate and stable estimate across the considered scenarios, while increasing to $K=64$ yields consistently tighter estimates that closely track the expected value.

\input{plots/compute_appendix}
\end{document}